\documentclass{article}

\usepackage{amsfonts}
\usepackage{amssymb}
\usepackage{stmaryrd}
\usepackage[makeroom]{cancel}
\usepackage{booktabs}
\usepackage[load-configurations=version-1]{siunitx}
\usepackage[english]{babel} 
\usepackage{graphicx}
\usepackage{adjustbox}
\usepackage{subcaption}
\usepackage{soul}
\usepackage{makecell}
\usepackage{colortbl}
\usepackage[dvipsnames,svgnames,x11names,hyperref]{xcolor}
\usepackage{comment}
\usepackage{longtable}
\usepackage{float}
\usepackage{makecell}
\usepackage{changepage}

\usepackage{algpseudocode,algorithm,algorithmicx}
 
\algrenewcommand\algorithmicindent{1em}%

\usepackage{microtype}
\usepackage{graphicx}
\usepackage{booktabs} 

\usepackage[accepted]{icml2020}

\usepackage{enumitem}
\setlist[itemize]{itemsep=0pt, parsep=0.5em, topsep=0pt}

\usepackage{amsmath}
\DeclareMathOperator*{\argmin}{arg\,min}
\DeclareMathOperator*{\argmax}{arg\,max}

\usepackage{amsthm}
\newtheorem{theorem}{Theorem}
\newtheorem{lemma}{Lemma}
\theoremstyle{definition}
\newtheorem{definition}{Definition}
\theoremstyle{remark}
\newtheorem*{remark}{Remark}
\let\oldproof\proof
\def\proof{\oldproof\unskip}

\usepackage{hyperref}
\hypersetup{
    colorlinks=true,
    citecolor=NavyBlue,
    linkcolor=NavyBlue,
    urlcolor=NavyBlue,
}
\usepackage[open,openlevel=1]{bookmark}

\usepackage[nameinlink,capitalise,noabbrev]{cleveref}
\crefname{equation}{Eqn.}{Eqn.}

\icmltitlerunning{Clinician-in-the-Loop Decision Making: RL with Near-Optimal Set-Valued Policies}

\begin{document}

\twocolumn[
\icmltitle{Clinician-in-the-Loop Decision Making: \\
Reinforcement Learning with Near-Optimal Set-Valued Policies}



\icmlsetsymbol{equal}{*}

\begin{icmlauthorlist}
\icmlauthor{Shengpu Tang}{cse}
\icmlauthor{Aditya Modi}{cse}
\icmlauthor{Michael W. Sjoding}{med,ihpi}
\icmlauthor{Jenna Wiens}{cse}
\end{icmlauthorlist}

\icmlaffiliation{cse}{Department of Electrical Engineering \& Computer Science, University of Michigan, Ann Arbor, US}
\icmlaffiliation{med}{Department of Internal Medicine, Michigan Medicine, University of Michigan, Ann Arbor, US}
\icmlaffiliation{ihpi}{Institute for Healthcare Policy \& Innovation, University of Michigan, Ann Arbor, US}

\icmlcorrespondingauthor{Shengpu Tang}{\href{mailto:tangsp@umich.edu}{tangsp@umich.edu}}
\icmlcorrespondingauthor{Jenna Wiens}{\href{mailto:wiensj@umich.edu}{wiensj@umich.edu}}

\icmlkeywords{Machine Learning, ICML, reinforcement learning, healthcare, clinical application, near-optimal}

\vskip 0.3in
]



\printAffiliationsAndNotice{}  

\begin{abstract}
  Standard reinforcement learning (RL) aims to find an optimal policy that identifies the best action for each state. However, in healthcare settings, many actions may be near-equivalent with respect to the reward (e.g., survival). We consider an alternative objective -- learning set-valued policies to capture near-equivalent actions that lead to similar cumulative rewards. We propose a model-free algorithm based on temporal difference learning and a near-greedy heuristic for action selection. We analyze the theoretical properties of the proposed algorithm, providing optimality guarantees and demonstrate our approach on simulated environments and a real clinical task. Empirically, the proposed algorithm exhibits good convergence properties and discovers meaningful near-equivalent actions. Our work provides theoretical, as well as practical, foundations for clinician/human-in-the-loop decision making, in which humans (e.g., clinicians, patients) can incorporate additional knowledge (e.g., side effects, patient preference) when selecting among near-equivalent actions. 
\end{abstract}

\section{Introduction}

In the standard RL setup, one aims to find an optimal policy, which identifies the action for each state that maximizes some discounted expected cumulative reward. However, in healthcare, the reward can be difficult to define; e.g., one might want to optimize for long-term quality of life vs. short-term stabilization of the symptoms (without treating the underlying disease process). Past work has augmented reward signals via reward shaping \citep{lizotte2012linear,raghu2017deepRL,nemati2016heparin}. Still, designing a single reward function that captures the goals and objectives across different individuals remains challenging. As a result, when applying RL in healthcare, survival is often used as the reward, since it represents a clear goal and is straightforward to measure \citep{raghu2017continuous,komorowski2018AI_Clinician,li2018ASTC}. 

While using survival as the reward signal can simplify the RL setup, we hypothesize that it induces many near-equivalent actions that could otherwise differ. For example, different doses of a drug might perform similarly in terms of keeping the patient alive, yet doses that are too large might cause severe side effects. In other cases, patients may opt for less invasive treatments, if they are likely to yield similar outcomes to more invasive ones. In such cases, learning a single best action and ignoring near-equivalent actions may be undesirable, as important considerations not captured through the reward signal can impact decision making. 

Thus, we consider the task of learning a ``set-valued policy'' (SVP), which returns a set of near-equivalent actions rather than a single optimal action. This allows for interaction between the clinician (or patient) and the decision support system (i.e., clinician/patient-in-the-loop). Such a setup provides clinicians/patients with an opportunity to incorporate additional information (e.g., patient preferences, adverse drug reactions, costs/availability of equipment) when choosing among near-equivalent actions. 

We study a particular formalization of SVPs introduced by \citet{fard2011non} where they consider the problem of computing SVPs (planning). In contrast to their exhaustive search based approach, which requires a model of the underlying environment, we propose a \textit{model-free} learning algorithm. The sequential nature of decisions makes learning such policies non-trivial, as near-equivalent actions should be not only similar in the short term (immediate reward), but also similar for all possible future trajectories (expected cumulative rewards). Our contributions are: 
\begin{itemize}[itemsep=-5pt, topsep=-2pt] %
    \item we propose a new algorithm based on temporal difference methods and a near-greedy heuristic for learning near-optimal set-valued policies,
    \item we investigate its convergence behavior and optimality using a worst-case analysis, providing theoretical guarantees in a directed acyclic graph (DAG) setting,
    \item we demonstrate the algorithm's behavior in both DAG and non-DAG synthetic environments, and
    \item on a clinical task, we demonstrate how the algorithm can help discover clinically meaningful near-equivalencies among treatment actions.
\end{itemize}

Our work provides both theoretical and practical foundations for learning near-optimal SVPs, and represents an important step towards clinician/human-in-the-loop decision support systems. Beyond applications in healthcare, this framework could also be applied to other domains involving human-machine-cooperative decision making, including intelligent tutoring systems and self-driving cars. The code to reproduce our experiments is available online\footnote{\url{https://gitlab.eecs.umich.edu/MLD3/RL-Set-Valued-Policy}}.

\section{Background}

We consider finite Markov decision processes (MDPs) defined by a tuple $(\mathcal{S}, \mathcal{A}, \mathcal{P}, \mathcal{R}, \gamma)$, where $\mathcal{S}$ is the set of states (finite or infinite), $\mathcal{A}$ is a finite set of actions, $\mathcal{P}: \mathcal{S} \times \mathcal{A} \times \mathcal{S} \rightarrow [0,1]$ defines the transition model with $p(s'|s,a)$ specifying the probability of moving from state $s$ to $s'$ given action $a$, and $\mathcal{R}: \mathcal{S} \times \mathcal{A} \rightarrow \mathbb{R}$ defines the reward function (can be stochastic in general) with $r(s,a)$ denoting the expected instantaneous reward obtained from taking action $a$ in state $s$, and discount factor $\gamma \in [0,1]$. 

\subsection{Set-Valued Policy}
In this work, we focus on set-valued policies (SVP), first formalized in \citet{fard2011non}\footnote{In \citet{fard2011non}, this was referred to as ``non-deterministic policy''; however, since the policy is indeed a deterministic mapping, we prefer the term ``set-valued policy'' to avoid confusion. }.

\begin{definition}
An SVP, $\pi$, is a function that maps each state to a non-empty subset of actions, $\pi: \mathcal{S} \rightarrow 2^\mathcal{A} \setminus \{\varnothing\}$. 
\label{def:svp}
\end{definition}

SVPs have been explored as a way to encode alternative choices \citep{fard2011non,lizotte2012linear} and to encapsulate an (approximately) exponential number of deterministic policies \mbox{\citep{lizotte2016multi}}. In particular, \citet{fard2011non} proposed a mixed-integer programming (MIP) formulation to find SVP solutions for a finite-horizon planning problem. However, their approach involves an exhaustive search with exponential complexity over the state and action spaces. In contrast, we propose a scalable model-free learning algorithm that does not require knowledge of the MDP. Moreover, our proposed approach allows the extent of near-optimality to be set \textit{a priori}, providing increased flexibility and optimality guarantees.

\subsection{Value Functions for SVPs}\label{sec:svp-value}
In contrast to the standard RL setting, here, a learning agent can suggest a set of actions. Thus, our notion of ``value'' must account for all possible policies consistent with the proposed set of actions, regardless of which proposed action is selected. To this end, we consider a worst-case analysis \citep{fard2011non}, where the value of a state $s$ is taken as the worst case over all the actions in the set $\pi(s)$. 
\begin{definition}\label{def:worst-case-vf}
The worst-case value functions of an SVP $\pi$ are defined as
\begin{align*}
    V^{\pi}(s) &= \min_{a\in \pi(s)} \{Q^{\pi}(s,a)\}, \\
    Q^{\pi}(s,a) &= r(s,a) + \gamma \mathbb{E}_{s'|s,a} \left[ V^{\pi}(s') \right].
\end{align*}
\end{definition}

\subsection{Near-Optimal SVPs}\label{sec:near-optimal}
We quantify the ``goodness'' of an SVP according to how far it is from the optimal value function, $V^*$. This gives rise to the definition of near-optimal SVPs. In some of the MDP literature, near-optimality has been formalized with an additive constraint: $\forall s, \ V^{\pi}(s) \geq V^*(s) - \epsilon$ \citep{satinder2002polytime}, which specifies a constant margin of sub-optimality across all states. This could lead to conservative action choices in some states, as $\epsilon$ is fixed but the magnitude of $V^{\pi}(s)$ and $V^{*}(s)$ could vary across different states. We argue that in healthcare settings this is not a suitable formalization, since at any point of decision making (in any particular state), the SVP should be near-optimal with respect to what we could achieve in that state; accepting the same value margin in a ``healthy'' state (larger value) vs. a ``sick'' state (smaller value) may not be desirable, because a margin that leads to acceptable outcomes in a ``healthy'' state can have a larger relative impact (perhaps devastating) in a ``sick'' state. We consider the additive near-optimality formalization as a baseline, providing the derivations of this setting in \cref{appx:additive}. However, given the limitations, we focus on a multiplicative constraint for near-optimality, which accounts for differences in values at different states. 

\begin{definition}\label{def:zeta-optimal}
An SVP $\pi$ is $\zeta$-optimal, with $\zeta \in [0,1]$, if
\[V^{\pi}(s) \geq (1-\zeta) V^*(s), \ \ \forall s \in \mathcal{S}.\]
\end{definition}
Here, $\zeta$ is a hyperparameter that defines the sub-optimality margin, quantifying the trade-off between action variety and optimality. 

\begin{remark}
Note that this definition requires $V^*(s)\geq 0, \forall s \in \mathcal{S}$. A sufficient (though not necessary) condition to ensure this is to enforce $r(s,a) \geq 0$. In experiments with clinical data (\cref{sec:experiments}), we discuss practical considerations to deal with problem domains having negative rewards. 
\end{remark}

As a na\"ive solution, one might construct $\pi$ as:
\[ \pi(s) = \{a : Q^{*}(s,a) \geq (1-\zeta) V^{*}(s)\}. \]

However, this construction fails to satisfy near-optimality. By using $Q^*(s,a) = r(s,a) + \gamma \mathbb{E}_{s'|s,a} V^*(s')$, the optimal Q-function, it assumes all future returns $V^*(s')$ are obtained following $\pi^*$ as opposed to $\pi$. During the roll-out of a policy, this fails to account for the fact that the future must be consistent with the policy. Alternatively, one might consider an exponentially large action space, $\tilde{\mathcal{A}} = 2^\mathcal{A} \setminus \{\varnothing\}$ and apply standard value-based methods to learn Q-values defined over $\mathcal{S} \times \tilde{\mathcal{A}}$, but analysis shows that this na\"ive approach defaults to the greedy optimal policy (\cref{appx:exp-action-space}). In the sections that follow, we propose a new approach to learn SVPs that does not violate the $\zeta$-optimal constraint in \cref{def:zeta-optimal} and focuses on value functions in the original action space $\mathcal{A}$.

\section{Methods}

We present an algorithm that jointly learns SVPs and their value functions. First, we provide two heuristics to construct near-optimal SVPs given the value functions. Using these heuristics as action selection strategies, we describe a variant of temporal difference (TD) learning and provide a theoretical analysis of its convergence and optimality. 

\subsection{Heuristics for Constructing Near-Optimal SVPs} \label{sec:construct}
To guarantee near-optimality, we start with a \textit{conservative} approach based on a loose lower bound of future returns. Then, we improve on this approach and propose a \textit{near-greedy} heuristic that leverages the learned policy. 

\textbf{Conservative.}\label{sec:conservative_svp}
Assuming the future follows a $\zeta$-optimal policy, one could construct $\pi$ as:
\begin{align}\label{eqn:conservative}
    &\pi(s) = \{a : \check{Q}_{\zeta}^{*}(s,a) \geq (1-\zeta)V^{*}(s)\},
\end{align}
where $\check{Q}_{\zeta}^{*}(s,a) = r(s,a) + \gamma (1-\zeta) \mathbb{E}_{s'|s,a}[V^*(s')] \leq Q^{*}(s,a)$ is the action-value function using a loose lower bound for near-optimal future returns. 

The conservative heuristic is consistent with the definition by \citet{fard2011non}; however, one key difference is that we provide an explicit way to construct a conservative SVP given an oracle for the optimal value function $V^*$. Though this will not violate the near-optimality bound, it may limit action diversity. 

To encourage action diversity while satisfying the near-optimality criteria, we can formulate a fixed-point equation for computing a near-optimal SVP. Recall that, in a standard RL setup, the optimal policy $\pi^*$ is the unique fixed-point solution to the following equation: $\pi^*(s) = \argmax_a Q^{\pi^*}(s,a)\; \forall s \in \mathcal{S}$ which applies a greedy action selection over optimal Q-values. For a near-optimal SVP, we seek the fixed-point solution to a similar equation with a near-greedy action selection. 

\textbf{Near-Greedy.} \label{sec:near-greedy_svp}
Consider the fixed-point solution to the following equation:
\begin{align}
    \pi(s) = \{a : Q^{\pi}(s,a) \geq (1-\zeta)V^{*}(s)\} \label{eqn:near-greedy}, 
\end{align}
where $Q^{\pi}(s,a)$ is the action-value function for policy $\pi$ as computed via \cref{def:worst-case-vf}. 

Depending on the dynamics of the MDP and the true optimal value function $V^*$, it is possible that no solution exists for \cref{eqn:near-greedy} (see \cref{appx:non-existence} for an example). Thus, for a general MDP, directly applying this heuristic might not generate the desired near-optimal SVP. In \cref{sec:theoretical}, we discuss the sufficient conditions for the existence and uniqueness of SVPs constructed according to \cref{eqn:near-greedy}. 

For the same optimality threshold $\zeta$, compared to the conservative SVP, the near-greedy SVP is more likely to contain more actions, due to its consideration of the policy-dependent worst-case future $V^{\pi}(s')$, rather than a loose lower bound $(1-\zeta)V^{*}(s')$, when deciding whether an action should be included.

\begin{figure}[h]
\scalebox{0.75}{%
\begin{minipage}{1.33\linewidth}%
\begin{algorithm}[H]
  \caption{TD learning for near-greedy $\zeta$-optimal SVP}\label{alg:near-greedy}
  \begin{algorithmic}[1]
    \State \textbf{Input:} step size $\alpha_t \in (0,1]$, \newline
    \hspace*{2.75em} optimal value function $V^*$ where $V^*(s') \geq 0\; \forall s\in \mathcal{S}$, \newline 
    \hspace*{2.75em} sub-optimality margin $\zeta \in [0,1]$.
    \State \textbf{Initialize} $Q(s,a)=0$ for all $s\in \mathcal{S}, a\in\mathcal{A}$
    \For{each episode}
        \State Initialize $s$
        \For{each step}
            \State Choose action $a$ using an exploratory policy (e.g., $\epsilon$-greedy)
            \State Take action $a$, observe $r$, $s'$
            \State $\pi(s') = \{a' : Q(s',a') \geq (1-\zeta)V^{*}(s')\}$
            \State \textbf{if} $\pi(s') \neq \varnothing$
            \State \hspace*{1em} $\displaystyle Q(s,a) \gets Q(s,a) + \alpha_t [r + \gamma \min_{a'\in \pi(s')} Q(s', a') - Q(s,a)]$
            \State \textbf{else}
            \State \hspace*{1em} $\displaystyle Q(s,a) \gets Q(s,a) + \alpha_t [r + \gamma \:\:\, \max_{a'\in \mathcal{A}}\:\:\, Q(s', a') - Q(s,a)]$
            \State $s \gets s'$
        \EndFor
    \EndFor
  \end{algorithmic}
\end{algorithm}
\end{minipage}
}%
\end{figure}

\subsection{Learning Near-Optimal SVPs}\label{sec:learn_svp}

As stated, a fixed-point solution to \cref{eqn:near-greedy} might not exist for a general MDP. However, the near-greedy construction can be used to devise a model-free learning algorithm in the TD-learning framework (\cref{alg:near-greedy}). Specifically, to compute the TD update target, we temporarily construct a candidate SVP $\pi(s')$ for the next state $s'$ using the current estimates of the Q-values (line 8). With $V^*(s') \geq 0$, when we have a non-empty set of near-optimal actions $\pi(s')$, we compute the update target by using the worst near-greedy action (line 10). Otherwise, when $\pi(s')$ is empty, we use the standard greedy target (line 12). Finally, the algorithm outputs estimates of $Q^{\pi}$, as well as the SVP $\pi$ constructed according to \cref{eqn:near-greedy}. Note that the algorithm requires the optimal value function $V^*$ as an input. In practice, one can run a separate Q-learning procedure to learn a good estimate of $Q^*$ and thus of $V^*$, or learn estimates of $Q^*$ and $Q^{\pi}$ simultaneously.

\subsection{Theoretical Analysis}\label{sec:theoretical}
Here, we discuss the sufficient conditions for which the near-greedy SVP exists and is unique and for which the near-greedy TD algorithm converges under the same conditions for Q-learning. For completeness, in \cref{appx:conservative}, we show that the conservative heuristic leads to a stochastic approximation algorithm that converges to a unique SVP solution for any MDP with non-negative rewards. 

\begin{theorem}\label{thm:greedy_exist}
The near-greedy $\zeta$-optimal SVP exists and is unique, if the MDP is a directed acyclic graph (DAG) with non-negative rewards. 
\end{theorem}

\begin{proof}
We show this by explicit construction of $\pi$. Note that the states in a DAG form a topological sort tree, where $s$ precedes $s'$ if and only if there is a transition from $s$ to $s'$. 

\begin{adjustwidth}{1em}{0pt}
\textit{Base case.} 
Consider every terminal state $s_\infty \in \mathcal{S}_{\infty}$, for which there is no immediate reward and no future time step starting from this state, i.e., $V^{\pi}(s_\infty)=0$. We set $\pi(s_\infty) = \mathcal{A}$ trivially. 

\textit{Inductive step.} 
Given a state $s$, consider its successor states $s'$. Assuming we have every $\pi(s')$ and their associated value functions satisfying $V^{\pi}(s') \geq (1-\zeta)V^*(s')$, we calculate $Q^{\pi}(s,a)$ for all $a\in \mathcal{A}$:

\vspace{-2em}
{\small\begin{align*}
    Q^{\pi}(s,a) &= r(s,a) + \gamma \mathbb{E}_{s'|s,a} V^{\pi}(s') \\
    &\geq r(s,a) + \gamma \mathbb{E}_{s'|s,a} (1-\zeta) V^{*}(s') \\
    &= \zeta r(s,a) + (1-\zeta) [r(s,a) + \gamma \mathbb{E}_{s'|s,a} V^{*}(s')] \\
    &= \zeta r(s,a) + (1-\zeta) Q^{*}(s,a) \\
    &\geq (1-\zeta) Q^{*}(s,a) .
\end{align*}}%
Using $Q^{\pi}(s,a)$, we construct $\pi(s)$ according to \cref{eqn:near-greedy}. Importantly, $\pi(s)$ is non-empty because the optimal action $a^*$ is always included. Given $\pi(s)$, $V^{\pi}(s) = \min_{a\in\pi(s)} Q^{\pi}(s,a) \geq (1-\zeta) V^{*}(s)$ satisfies the near-optimality constraint by construction.
\end{adjustwidth}

Enumerating the states in reverse topological order starting from the terminal states, we follow this process until we have $\pi(s)$ for all $s \in \mathcal{S}$. Each step of the construction process is unique, hence overall the policy is unique. 
\end{proof}%
Since the proof relies on the topological ordering of states in a DAG MDP, it holds when we are only given $V^*$, as we can obtain $V^{\pi}$ from the base case for terminal states and from the inductive step for non-terminal states. The above theorem provides a \textit{sufficient} condition for a near-greedy SVP to exist: the environment is a DAG with non-negative rewards. An MDP with cycles and/or negative rewards may still have a near-greedy $\zeta$-optimal SVP, depending on the sub-optimality margin $\zeta$ and the MDP parameters. 

Similarly, we can show the following convergence result for \cref{alg:near-greedy} (proof provided in \cref{appx:near-greedy-convergence-proof}). 

\begin{theorem}\label{thm:greedy_converge}
The near-greedy TD algorithm converges to the unique solution if the MDP is a DAG with non-negative rewards, under the same convergence conditions for regular TD learning: rewards have bounded variance, each $(s,a)$ is updated infinitely many times, $\sum_{t} \alpha_{t}=\infty$, $\sum_{t} \alpha_{t}^{2}<\infty$. 
\end{theorem}

\begin{remark}
The near-greedy heuristic from \cref{sec:near-greedy_svp} can be used as the policy improvement step (replacing the greedy action selection) in any value-based generalized policy iteration algorithm. For instance, when we are given the MDP model, one can derive a version of near-greedy value iteration with similar theoretical guarantees on DAG environments. In that case, it is most efficient to learn the value functions in reverse topological order of the states (as per the proof). However, in the more general case (where the underlying MDP is unknown, or when we have a continuous state space), we require a different approach. In particular, here we present an algorithm based on TD learning. Though the theoretical analyses only hold for DAG environments, in our experiments we will demonstrate the algorithm's behavior in the more general setting of non-DAG environments. 
\end{remark}

\section{Experiments} \label{sec:experiments}
We consider a set of experiments to i) demonstrate the ability of the proposed algorithm to learn SVPs for different values of $\zeta$ and ii) characterize the algorithm's empirical convergence behavior. Throughout, we compare to alternative approaches to learning SVPs (described in \cref{sec:baseline}).

\begin{figure}[h]
\begin{tabular}{m{0.48\hsize}m{0.48\hsize}}
\begin{subfigure}[b]{\linewidth}
\centering \includegraphics[width=\linewidth]{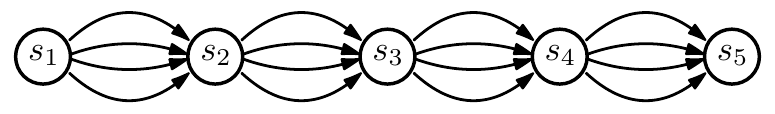}
\caption{\textsf{Chain-5}}\label{fig:chain5}
\end{subfigure}
\par\bigskip
\begin{subfigure}[b]{\linewidth}
\centering \includegraphics[width=\linewidth]{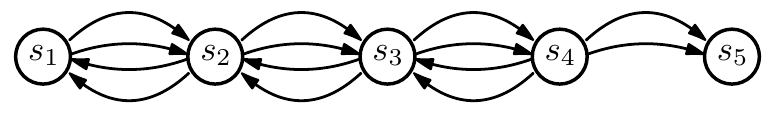}
\caption{\textsf{CyclicChain-5}}\label{fig:cycle5}
\end{subfigure}
&
\centering
\begin{subfigure}[b]{\linewidth}
\centering \includegraphics[width=0.5\linewidth]{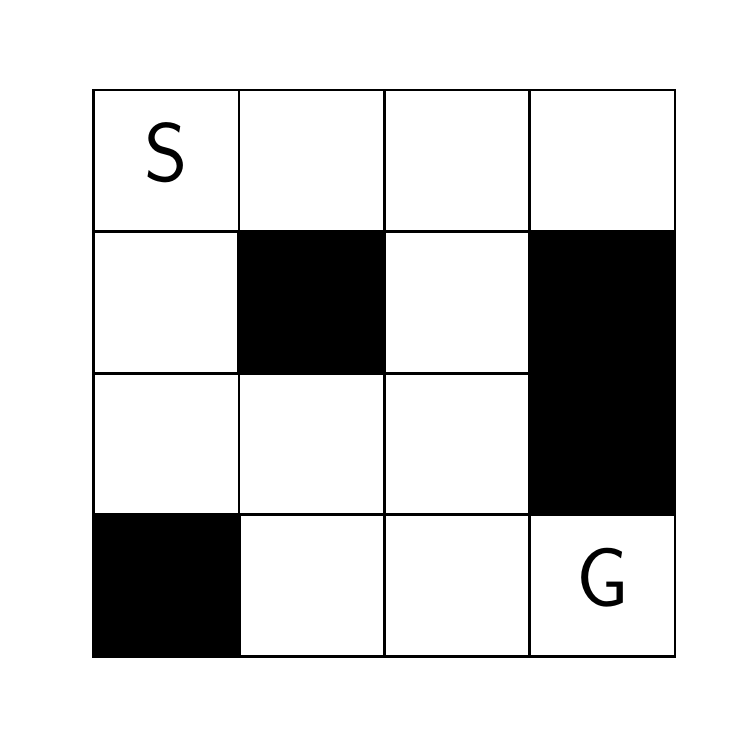}
\caption{\textsf{FrozenLake-4x4}}\label{fig:lake}
\end{subfigure}
\end{tabular}
\caption{Synthetic environments include (a) a simple DAG, and (b)(c) two non-DAG environments. }\label{fig:synth_env}
\end{figure}

\subsection{Environments} \label{sec:environments}
We consider both synthetic and real environments. 

\textsf{Chain-$k$} (\cref{fig:chain5}). First, as a sanity check, we consider a DAG with $k$ sequentially connected states, with $s_1$ the starting state and $s_k$ the terminal state. There are four actions at each state $s_i$ that transition to $s_{i+1}$, and intermediate rewards $r_1, r_2, r_3, r_4 \in \{0, 0.01, 0.02, 0.03, 0.04, 0.05\}$ that are predetermined, but randomly assigned. Transitions reaching the terminal state result in an additional reward of $+1$. In this simple setting, we test the ability of the proposed algorithm to identify near-equivalent actions. 

Next, we investigate the empirical performance of the algorithm in two non-DAG settings: \textsf{CyclicChain-$k$} and \textsf{FrozenLake}. These represent more complex environments, testing the generalizability of the proposed algorithm. 

\textsf{CyclicChain-$k$} (\cref{fig:cycle5}): 
An extension of \textsf{RandomWalk} \citep{sutton2018RL}, similar to \textsf{Chain-$k$}, except that of the four actions from $s_i$, two lead to $s_{i-1}$ and two lead to $s_{i+1}$. We set $\gamma < 1$ to encourage the agent to reach the terminal state quickly and avoid cycling. 

\textsf{FrozenLake} (\cref{fig:lake}). This is a standard discrete space path-finding problem from OpenAI Gym \citep{openai-gym}. The agent controls the movement of a character in a grid world. Some tiles of the grid lead the agent to fall into the water. The agent is rewarded $+1$ for finding a path to a goal tile. We used the standard $4 \times 4$ and $8\times 8$ maps. Multiple paths exist from the starting tile to the goal tile; these paths take the same number of steps and are thus equivalent. To introduce near-equivalent actions, we added a small reward to all actions from every non-terminal state so that the four actions vary slightly in value. For the $4 \times 4$ map, we added $r_i \in \{0.01, 0.02, 0.03, 0.04\}$, while for the $8 \times 8$ map, we added $r_i \in \{0.001, 0.002, 0.003, 0.004\}$. 

Finally, we explore a challenging clinical task using observational patient data. In contrast to the environments above, in this setting, i) we do not have access to the underlying MDP, ii) transitions are stochastic, and iii) the reward signal is sparse. 

\textsf{MIMIC-sepsis}. This is a previously studied RL task in the healthcare domain, in which the goal is to learn optimal treatment strategies for patients with sepsis in the ICU. Sepsis is defined as severe infection leading to life-threatening organ dysfunction and is one of the leading causes of mortality in hospitals \citep{gotts2016sepsis, liu2014hospital}. While a lot of work has focused on sepsis prediction \citep{henry2015targeted, reyna2019early, bedoya2020machine}, the management of intravenous (IV) fluids and vasopressors in sepsis treatment still represents a key clinical challenge \citep{byrne2017fluid}. We based our analysis, in part, on the setup described in \citet{komorowski2018AI_Clinician}, and used the same data and preprocessing steps, outlined below. Patient data are 48-dimensional time series (\cref{appx:features}) collected at 4h intervals, consisting of measurements from 24h preceding until 48h following the time of sepsis onset. Similar to in \citet{komorowski2018AI_Clinician}, we consider 750 discrete health states obtained from clustering the training set using k-means. Additionally, 2 terminal states are added to represent death and discharge. Actions pertain to treatment decisions in each 4h interval, representing total volume of IV fluids and amount of vasopressors administered. Though we consider the same number of discrete actions (25), the corresponding IV fluid doses differ substantially from those considered by \citet{komorowski2018AI_Clinician}. Specifically, we updated the five levels of IV fluids to use the following bins [0, $<$500mL, 500mL$-$1L, 1$-$2L, $>$2L] to represent more clinically relevant fluid boluses. We made this modification based on feedback from a critical care physician. Furthermore, the actions available at each state, $A(s)$, are restricted to those observed $\geq 5$ times (in training data; the most frequent action is used if no action occurs $\geq 5$ times). Rewards are sparse and only assigned at the end of each trajectory: $+100$ for survival (and discharge), $-100$ for in-hospital death; all intermediate rewards are $0$. In applying \cref{alg:near-greedy}, when $V^*(s)<0$ (due to negative rewards), we fall back to the greedy update target in line 12. $\gamma$ is set to $0.99$ to place nearly as much importance on late deaths as early deaths. Applying the specified inclusion and exclusion criteria \citep{komorowski2018AI_Clinician} to the MIMIC-III database \citep{mimic3}, we identified a cohort of 20,940 patients with sepsis (\cref{tab:sepsis}). The cohort was split into 70\% training, 10\% validation and 20\% test. 
\begin{table}[h]
    \caption{Cohort statistics of \textsf{MIMIC-sepsis} patients.}
    \label{tab:sepsis}
    \centering
\scalebox{0.8}{
\begin{tabular}{lrrrr}
\toprule
      &     N & \% Female & Mean Age & \makecell[r]{\small Mean Hours\\ \small in ICU} \\
\midrule
     Survivors & 18,057 &    44.3\% &     64 &         56.6 \\
 Non-survivors &  2,883 &    42.7\% &     69 &         60.9 \\
\bottomrule
\end{tabular}}
\end{table}

In an effort to explore the stability of training with function approximation, instead of the tabular lookup algorithm, we implemented a linear approximator for the Q-function with a one-hot state feature encoding (based on the clustering results), where we aimed to minimize the mean squared TD error \citep{sutton2018RL}. This setup allows the implementation to be readily extended to any linear (or possibly non-linear) approximation of the Q-function.

\subsection{Baselines} \label{sec:baseline}
We compare our proposed SVP learning algorithm based on a near-greedy heuristic to one based on a conservative heuristic. In addition, we compare to three other baselines:
\begin{itemize}
\item \textit{$Q^*$-based} (\cref{sec:near-optimal}). This approach assumes that the future follows an optimal policy, which can result in $\pi(s)$ including arbitrarily bad actions, especially in complex environments with long horizons.
\item \textit{$Q$-based}. In \cref{alg:near-greedy}, replace the policy construction step (line 8) of update rule with:
\begin{equation*}
    \pi(s) = \{a : Q(s,a) \geq (1-\zeta)V(s)\}, 
\end{equation*}
where $V(s) = \max_{a\in \mathcal{A}} Q(s,a)$ replaces $V^*(s)$. This is similar to the near-greedy algorithm, except that it uses $V$, the worst-case value of the SVP as the baseline for near-optimality, instead of $V^*$ in the definition of $\zeta$-optimal, so the optimality constraint may be violated. 
\item \textit{Additive}. After learning $Q^*$, we construct $\pi$ following the additive constraint definition in \cref{appx:additive}: 
\begin{equation*}
    \pi(s) = \{a : Q^{\pi}(s,a) \geq V^*(s) - \epsilon\}, 
\end{equation*}
where $\epsilon = \zeta(1-\gamma)\|V^*\|_\infty$. Using this selection criteria, we can guarantee that $\forall s:$ $V^\pi(s) \ge V^*(s) - \zeta \|V^*\|_\infty$.
\end{itemize}

In addition to comparing to these model-free alternatives, we compare to the MIP approach proposed by \citet{fard2011non} in \cref{appx:MIP-baseline} (their method requires knowledge of the MDP model). Another alternative involves first learning a (parameterized) probabilistic policy $\pi_{\theta}(a|s)$ and then converting it to an SVP by thresholding based on the action probabilities, $\pi(s) = \{a: \pi_{\theta}(a|s) \geq \tau \}$. Relevant methods to learn probabilistic policies include: diversity-inducing policy gradient \citep{masood2019DIPG} and maximum entropy approaches (e.g., soft actor-critic \citep{haarnoja2018SAC}). However, such approaches rely on probabilistic assumptions when computing Q-values and require one to choose a probability threshold $\tau$, which does not align with our definition of near-optimality.

\subsection{Evaluation}
We perform qualitative evaluations by inspecting the learned sets of near-equivalent actions. In particular, we visualize the different routes induced by learned SVPs in the \textsf{FrozenLake} environment, as well as near-equivalent treatment actions in \textsf{MIMIC-sepsis}. 

To evaluate the quality of learned policies, we use standard Q-learning to establish the optimal Q-function $Q^*$ as the baseline for deciding near-optimality. For synthetic environments, we characterize SVPs in terms of:
\begin{itemize}
    \item Average policy size: the expected number of actions to which the SVP maps a state, $\frac{1}{|\mathcal{S}|} {\scriptstyle \sum\limits_{s\in \mathcal{S}}} |\pi(s)|$. A larger average policy size means more actions are considered near-equivalent within the sub-optimality margin $\zeta$. 
    \item Worst-case near-optimality: we consider the largest deviation from $V^*$ among all states, $1-{\scriptstyle \min\limits_{s\in\mathcal{S}}} \frac{V^{\pi}(s)}{V^*(s)}$. This represents to what extent, in the worst-case, optimality is sacrificed in exchange for more choice. The value of each state is found by running a modified version of policy evaluation for the returned SVP (see \cref{appx:policy-evaluation} for pseudo-code and convergence result). 
\end{itemize}

Evaluating the learned policies on the real data task \textsf{MIMIC-sepsis} presents significant challenges due to the present limitations of off-policy evaluation methods \citep{imbens2015causal,thomas2016data,gottesman2018evaluating}. However, our main focus is on testing whether or not we can learn reasonable near-equivalent actions, rather than on learning the best policy for sepsis treatment. Still, for completeness, we evaluated the quality of the learned policies using multiple off-policy evaluation estimators and considered a non-deterministic variation of the learned policies \citep{gottesman2018evaluating}. For both the estimated optimal policy and the learned SVPs, we follow \citet{komorowski2018AI_Clinician} and evaluate softened policies: 99\% probability is distributed among actions in the recommended set; the remaining 1\% is distributed to non-suggested actions. This allows us to use nearly all sample trajectories in the test set, maintaining a large effective sample size. We applied two off-policy evaluation methods, doubly-robust estimator (DR) \citep{jiang2015doubly} and weighted doubly robust estimator (WDR) \citep{thomas2016data}, to evaluate the softened policies derived from the learned policies, computing empirical error bars based on 1,000 bootstraps of the test set. 

\section{Results}
Across environments, our proposed approach is able to learn SVPs and discover near-equivalent actions. Empirically, we observe good convergence with near-optimal behavior under reasonable settings of $\gamma$ and $\zeta$ even within non-DAG environments. The near-greedy algorithm outperforms the baselines, while the learned SVPs induce diverse behavior and meaningful alternative treatment recommendations.

\begin{figure}[h]
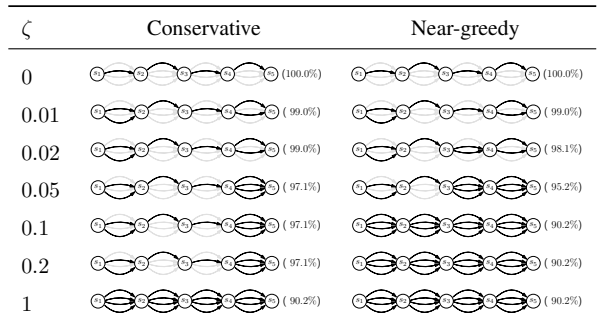

\begin{center}
\renewcommand*{\arraystretch}{1.2}
\scalebox{0.825}{
\begin{tabular}{lcc}
\toprule
$\zeta$ & Conservative & Near-greedy \\
\midrule
$0$ 
& \adjustimage{scale=0.4,valign=m,padding=0em 2pt}{{fig/chain_5_conservative_zeta_0.0}.pdf} 
& \adjustimage{scale=0.4,valign=m,padding=0em 2pt}{{fig/chain_5_near-greedy_zeta_0.0}.pdf} 
\\
$0.01$ 
& \adjustimage{scale=0.4,valign=m,padding=0em 2pt}{{fig/chain_5_conservative_zeta_0.01}.pdf}
& \adjustimage{scale=0.4,valign=m,padding=0em 2pt}{{fig/chain_5_near-greedy_zeta_0.01}.pdf}
\\
$0.02$ 
& \adjustimage{scale=0.4,valign=m,padding=0em 2pt}{{fig/chain_5_conservative_zeta_0.02}.pdf} 
& \adjustimage{scale=0.4,valign=m,padding=0em 2pt}{{fig/chain_5_near-greedy_zeta_0.02}.pdf}
\\
$0.05$ 
& \adjustimage{scale=0.4,valign=m,padding=0em 2pt}{{fig/chain_5_conservative_zeta_0.05}.pdf} 
& \adjustimage{scale=0.4,valign=m,padding=0em 2pt}{{fig/chain_5_near-greedy_zeta_0.05}.pdf}
\\
$0.1$ 
& \adjustimage{scale=0.4,valign=m,padding=0em 2pt}{{fig/chain_5_conservative_zeta_0.1}.pdf} 
& \adjustimage{scale=0.4,valign=m,padding=0em 2pt}{{fig/chain_5_near-greedy_zeta_0.1}.pdf}
\\
$0.2$ 
& \adjustimage{scale=0.4,valign=m,padding=0em 2pt}{{fig/chain_5_conservative_zeta_0.2}.pdf} 
& \adjustimage{scale=0.4,valign=m,padding=0em 2pt}{{fig/chain_5_near-greedy_zeta_0.2}.pdf}
\\
$1$ 
& \adjustimage{scale=0.4,valign=m,padding=0em 2pt}{{fig/chain_5_conservative_zeta_1.0}.pdf} 
& \adjustimage{scale=0.4,valign=m,padding=0em 2pt}{{fig/chain_5_near-greedy_zeta_1.0}.pdf}
\\
\bottomrule
\end{tabular}
}
\end{center}
\caption{SVPs learned by the near-greedy and conservative algorithms on \textsf{Chain-5} at different $\zeta$s. Parenthesized percentages denote the worst-case near-optimality.} \label{fig:chain5-policy}
\end{figure}

\begin{figure}[h]
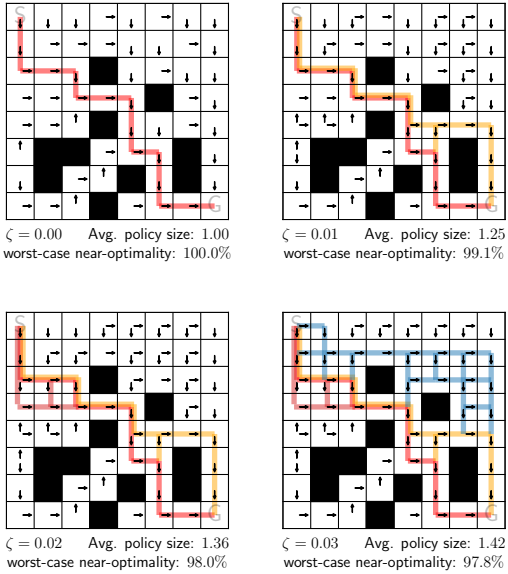

\centering
\begin{tabular}{cc}
\adjustimage{scale=0.5,valign=m,padding=0em 0pt}{{fig_lake/lake8x8.zeta=0.0}.pdf} 
& \adjustimage{scale=0.5,valign=m,padding=0em 0pt}{{fig_lake/lake8x8.zeta=0.01}.pdf}
\\
\\
\adjustimage{scale=0.5,valign=m,padding=0em 0pt}{{fig_lake/lake8x8.zeta=0.02}.pdf}
& \adjustimage{scale=0.5,valign=m,padding=0em 0pt}{{fig_lake/lake8x8.zeta=0.03}.pdf}
\end{tabular}
\caption{On the \textsf{FrozenLake-8x8} environment, at different levels of near-optimality, $\zeta$, the near-greedy algorithm learns different but near-equivalent routes to the goal, \textsf{G}. The learned SVPs conform to the near-optimality margin in the worst-case for the range of $\zeta$ values shown. } \label{fig:lake-routes}
\end{figure}

\subsection{Applied to Synthetic Data}\label{sec:algo-behavior}

In the simple DAG environment, \textsf{Chain-5} with $\gamma=0.9$, we recover near-equivalent actions. When $\zeta=0.01$ is small, $\pi(s_1)$ and $\pi(s_4)$ for both the near-greedy and conservative heuristics include two actions (\cref{fig:chain5-policy}). As $\zeta$ increases, so does the number of actions, resulting in diversity and choice among near-equivalent actions. As expected, for the same $\zeta$, the conservative approach contains fewer actions compared to the near-greedy SVP. Moreover, for a wide range of $\zeta \geq 0.05$, the conservative approach underestimates the near-optimality, failing to produce a diversity of choice. 

On \textsf{FrozenLake-8x8} with $\gamma=0.9$, our proposed near-greedy algorithm discovers different routes to the goal tile (\cref{fig:lake-routes}). Due to subtle differences in instantaneous rewards (see \cref{sec:environments}), these routes are not exactly the same, but are near-equivalent. 

\textbf{How does the algorithm perform empirically on non-DAGs?}
Here, we investigate the empirical performance of the near-greedy algorithm in non-DAG environments, namely, \textsf{CyclicChain-5} and \textsf{FrozenLake-4x4}. We substituted the near-greedy heuristic as the policy improvement step in the value iteration algorithm and monitor whether the SVP (derived from the learned value estimate $V$) stablizes towards the end of training. For small $\gamma$ and small $\zeta$, the near-greedy algorithm demonstrates good convergence (\cref{fig:converge}). Under settings with large $\gamma$ or large $\zeta$ (close to $1$), the algorithm displays instability. For large $\gamma$, the effective horizon is longer and the worst-case value function must account for a longer future. When $\zeta$ is large, we permit more sub-optimal actions, and the local effect of adding an action to $\pi(s)$ could be so large that the action is no longer near-optimal, resulting in oscillating behavior and non-convergence. Notably, for the region of parameter values where the algorithm converges, the learned SVPs are non-trivial solutions (they include more than one action for certain states), as indicated by lighter colors in \cref{fig:converge}. Empirically, for many operating regimes with reasonable settings of $\gamma$ and $\zeta$ (e.g., small $\zeta$ for near-optimality), the near-greedy algorithm exhibits good convergence behavior. 

\begin{figure}[h]
\begin{subfigure}[b]{.5\linewidth}
\centering \includegraphics[width=0.8\linewidth, trim=0 0 75 0, clip]{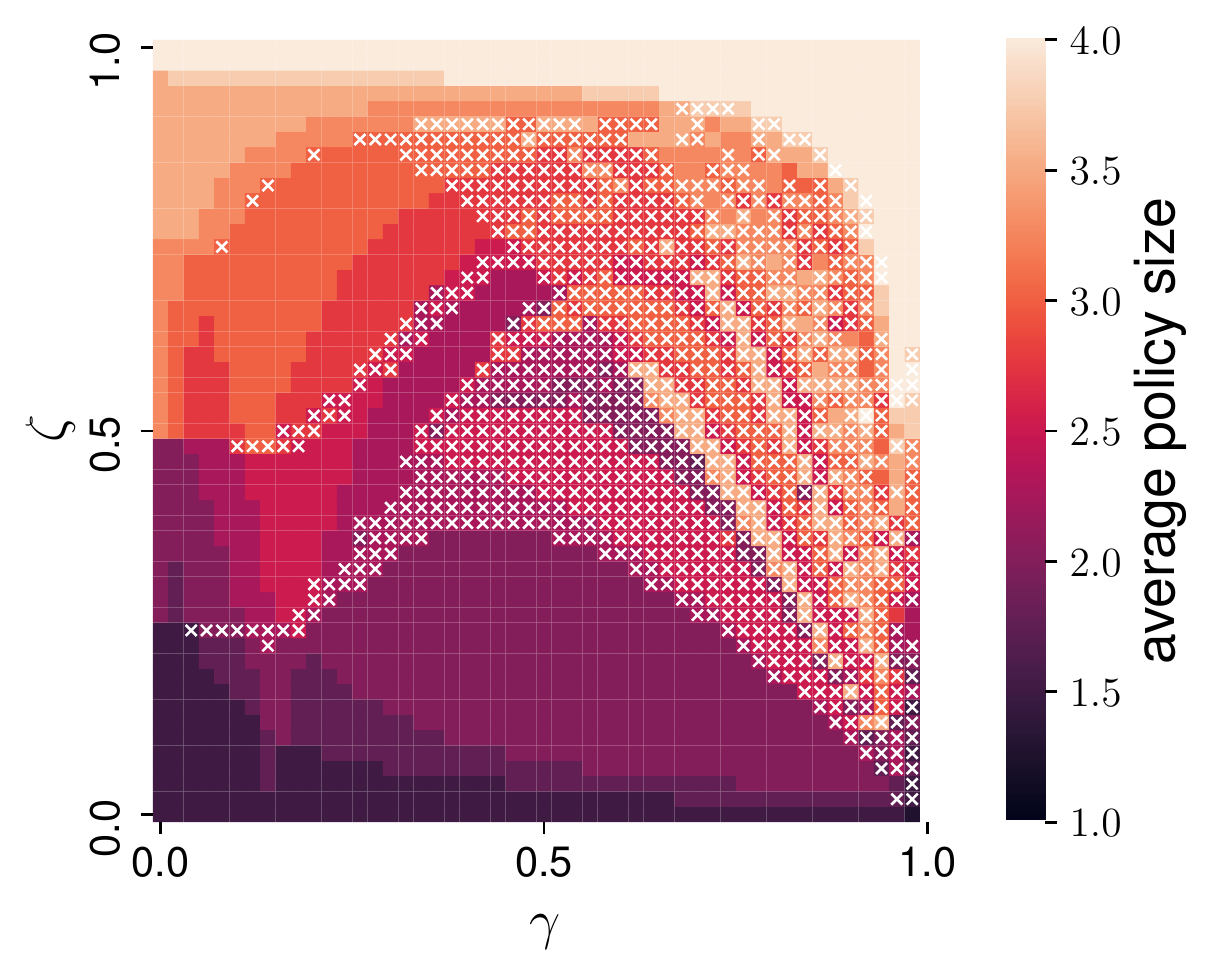}
\caption{\textsf{CyclicChain-5}}\label{fig:converge-a}
\end{subfigure}%
\begin{subfigure}[b]{.5\linewidth}
\centering \includegraphics[width=\linewidth]{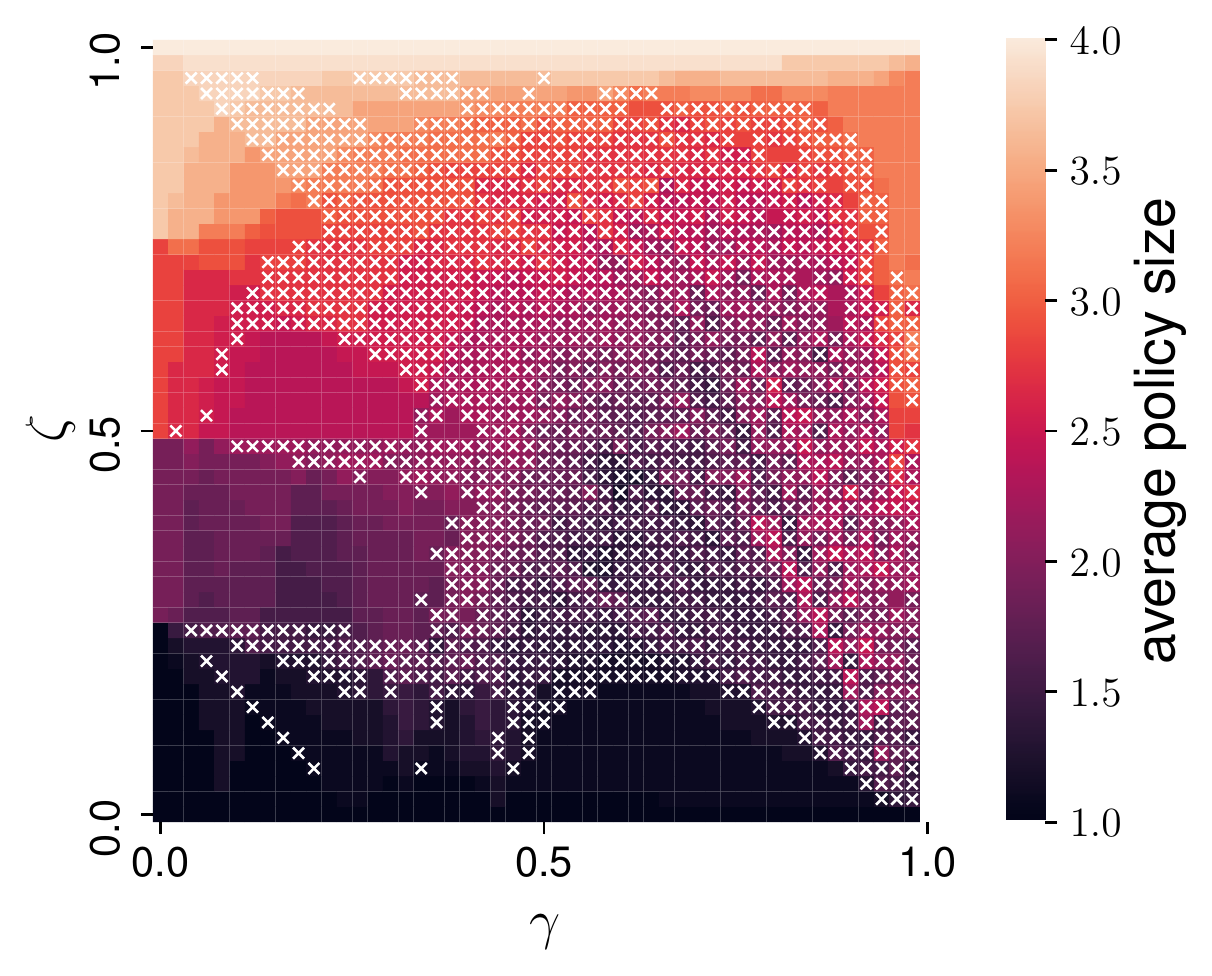}
\caption{\textsf{FrozenLake-4x4}}\label{fig:converge-b}
\end{subfigure}
\caption{Empirical behavior of the near-greedy value iteration algorithm, visualized as a heatmap of policy size at the end of training. White $\times$'s indicate non-convergence (i.e., the SVP derived from the learned $V$ does not stablize towards the end of training). In regions of the parameter space where the algorithm converges, we are able to discover non-trivial solutions (learned policies include more than one action for some states). }\label{fig:converge}
\end{figure}

\begin{figure}[h]
\begin{subfigure}[b]{.5\linewidth}
\centering \includegraphics[width=\linewidth]{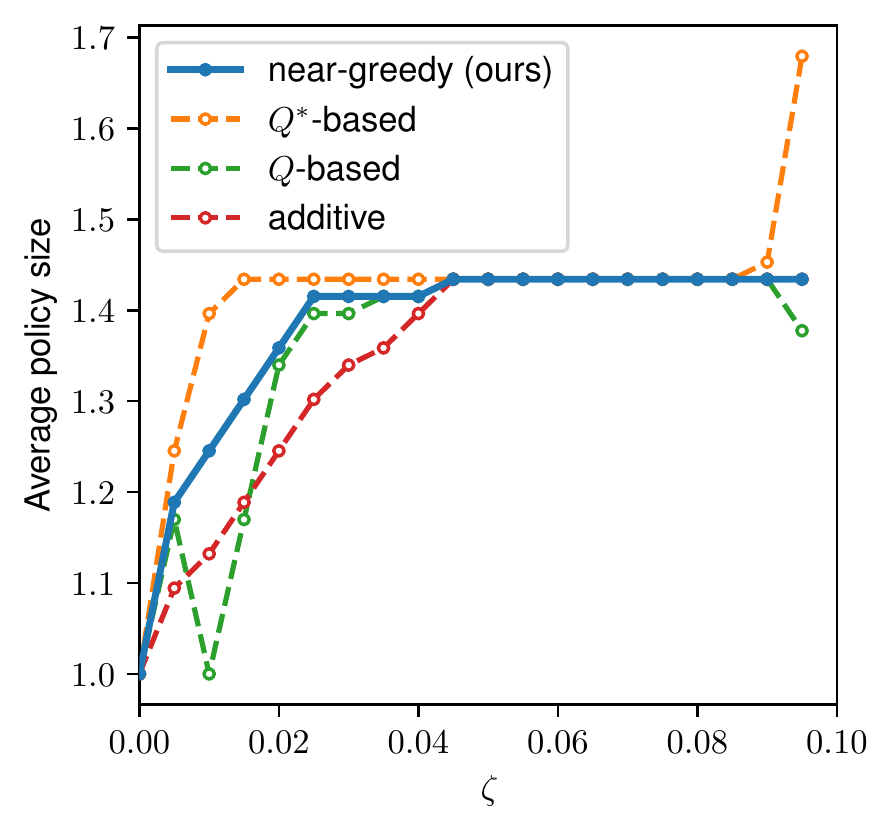}
\caption{}\label{fig:avg_pol_size}
\end{subfigure}%
\begin{subfigure}[b]{.5125\linewidth}
\centering \includegraphics[width=\linewidth]{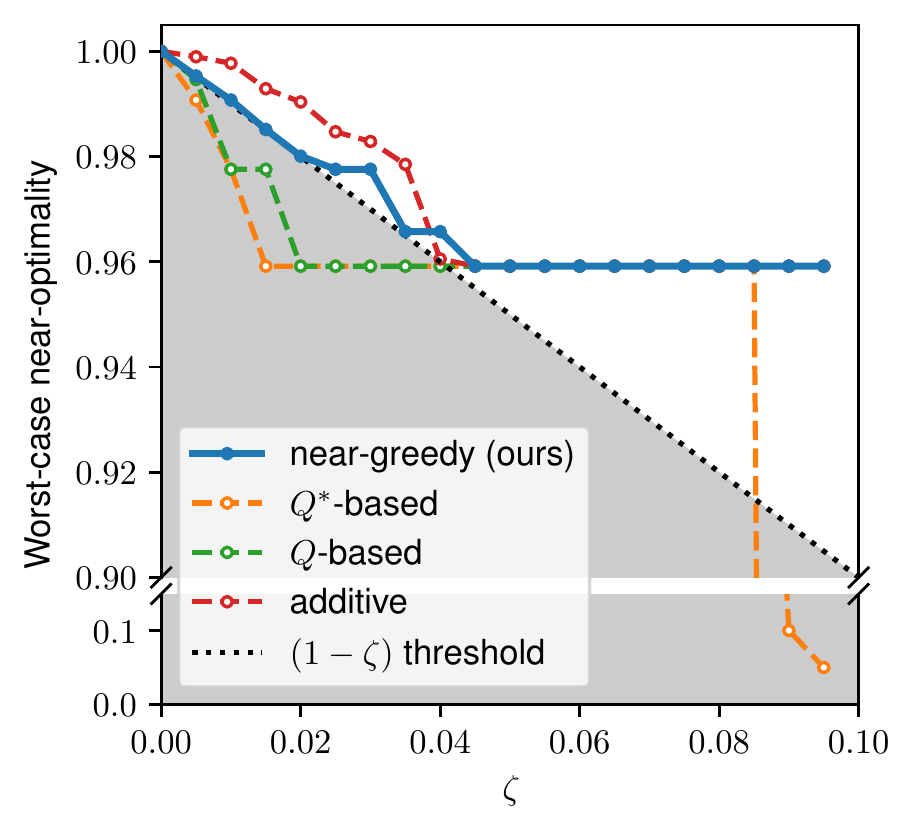}
\caption{}\label{fig:worst_near_optimal}
\end{subfigure}%
\caption{Comparison of average policy size and worst-case near-optimality of SVPs for different values of $\zeta$ on \textsf{FrozenLake-8x8}. }\label{fig:baseline}
\end{figure}

\textbf{How does the proposed algorithm compare to the alternatives?}
Using the \textsf{FrozenLake-8x8} environment with $\gamma=0.9$, we compare the proposed approach to the three baselines in terms of average policy size (\cref{fig:avg_pol_size}) and worst-case near-optimality (\cref{fig:worst_near_optimal}). We focused on small values of $\zeta$, given the goal is to learn near-optimal behavior. For this particular environment, the conservative algorithm completely fails to discover near-equivalent actions. As $\zeta$ increases, the near-greedy finds solutions of comparable policy size to the baselines. However, holistically, near-greedy performs best because it respects the predefined optimality threshold (lies to the upper right of the shaded region) while maximizing the policy size. For $0< \zeta < 0.04$, both \textit{$Q^*$-based} and \textit{$Q$-based} baselines violate the predefined near-optimality threshold while only yielding marginally larger policy sizes. For $\zeta \geq 0.9$, the \textit{$Q^*$-based} approach finds an arbitrarily bad solution with a worst-case near-optimality of $< 10\%$. This simple environment illustrates the importance of considering a worst-case future during learning and accounting for the inter-dependence of including/excluding actions at different states. On this environment, we used $\|V^*\|_\infty$ to define the margin $\epsilon$ for \textit{additive}, and the solution includes fewer actions and does not make good use of the allowable sub-optimality compared to \textit{near-greedy} for $\zeta \leq 0.04$ (for other values of $\zeta$ it leads to the same behavior as \textit{near-greedy}).

\subsection{Applied to Real Clinical Data}

For the \textsf{MIMIC-sepsis} task, we apply our proposed approach with a linear function approximator by first running Q-learning to estimate $Q^*$, and then running the near-greedy algorithm with various values of $\zeta$. During training, each episode is generated by randomly sampling a patient trajectory from the training set (with replacement). Given the complexity of this environment, to improve convergence, we exponentially decay the step size $\alpha$ every $1,000$ episodes. We train the RL agent for $1,000,000$ episodes, after which TD errors stabilize and the estimated Q-values reach plateaus. We consider $\zeta$ values from $0.0$ to $0.1$ to obtain SVPs that are near-optimal. 

For illustrative purposes, we focus on an intermediate value of $\zeta=0.05$, where $50\%$ of the states are mapped to more than one near-optimal action. On the test set, the estimated value of the softened policy derived from $\pi$ is within approximately $5$--$8\%$ of the estimated optimal value (\cref{tab:OPE}), which closely matches with the optimality margin $\zeta=0.05$, given the complexity of this task and the noise in the data. 

To better understand the near-equivalence relationships among actions, for each `optimal' action, we count the number of times every other action is suggested by the learned SVP. These counts are aggregated over all states in the test set and normalized by the maximum count (red numbers next to each cell), visualized as a heatmap in \cref{fig:mimic_actions}. In interpreting these results, we worked closely with our co-author, MWS, a critical care physician who frequently treats patients with sepsis. 

By far, the most common actions correspond to IV fluids alone (region A), i.e., no vasopressors are used. We observe that actions with similar amounts of IV fluids are often considered near-equivalent. Since the differences in fluid volumes across these actions (adjacent cells) are non-trivial, we hypothesize that grouping those actions together makes sense for some (but not all) patient states. Similarly, in region B where actions correspond to a low dose of IV fluids with various amounts of vasopressors, actions with similar vasopressor doses are often considered near-equivalent. We observe that the near-equivalent action sets are not always contiguous across IV/vasopressor doses. This is due, in part, to the fact that we restricted the action space $A(s)$ for each state to only those that were taken frequently. Interestingly, in region C, actions with very high vasopressor and IV fluid doses are considered near-equivalent to the null action (i.e., do nothing). Typically, only the sickest patients are prescribed the highest doses. The near-equivalence of this action with the null action may be due to the fact that these patients are so critically ill that doing `everything' or `nothing' leads to a similar outcome. 

\begin{table}[t]
    \caption{Value estimates of the learned SVPs on \textsf{MIMIC-sepsis}, with standard errors from 1,000 bootstraps of the test set. Effective sample size (measured as: usable trajectories / total number of trajectories, of the test set): 2,801$/$4,189. }
    \label{tab:OPE}
    \centering
    \vspace{0.15em}
\scalebox{0.75}{
\begin{tabular}{lcc}
\toprule
 & \multicolumn{2}{c}{observed returns of test set} \\
clinician & \multicolumn{2}{c}{73.1 $\pm$ .97} \\
\midrule
$\zeta$  & DR & WDR \\
\midrule
0, $\pi^*$ & 91.6 $\pm$ .31 & 92.2 $\pm$ .23 \\
0.05 & 84.3 $\pm$ .63 & 89.7 $\pm$ .32 \\
\bottomrule
\end{tabular}}
\end{table}

\begin{figure}[t!]
    \centering
    \includegraphics[width=\linewidth]{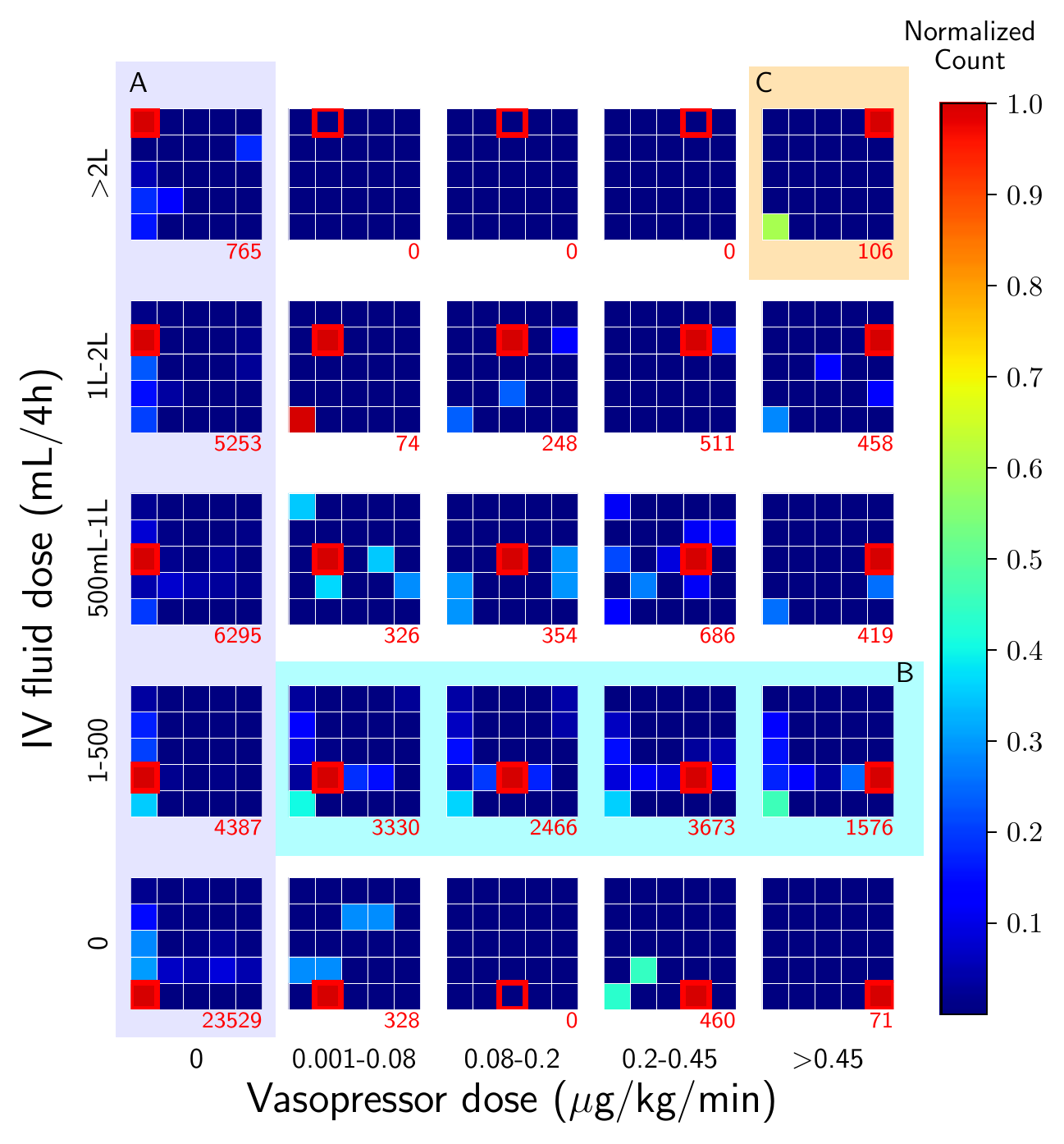}
    \caption{Near-equivalent relationships among actions in SVPs on the \textsf{MIMIC-sepsis} task with $\zeta=0.05$. Both the outer grid and the inner grid contain the 25 actions, corresponding to 5 doses of vasopressor and 5 doses of IV fluids. Red numbers indicate the frequency that each action is recommended as optimal. }
    \label{fig:mimic_actions}
\end{figure}

\begin{figure}[t!]
    \centering
    \includegraphics[width=0.6\linewidth]{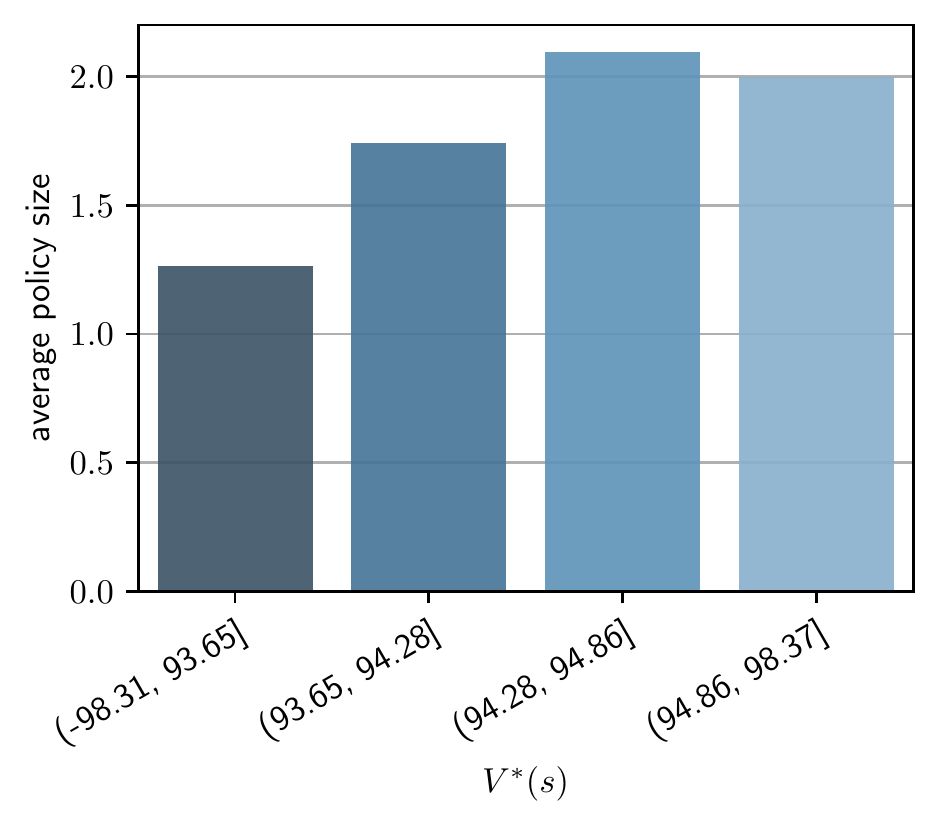}
    \caption{Average size of $\pi(s)$ at $\zeta=0.05$ for states with different $V^*(s)$ values (grouped into quartiles) on \textsf{MIMIC-sepsis}. }
    \label{fig:mimic-psize}
\end{figure}

In \cref{fig:mimic_actions}, we computed the `average' near-equivalencies across all states. However, different sets of actions could be considered near-equivalent for different states. To better understand how action near-equivalencies might differ for different group of patients, we grouped the states into quartiles based on the $V^*$ values and calculated the average policy size (\cref{fig:mimic-psize}). Compared to `sicker' states (with a lower value), the `healthier' states (with a higher value) had on average more near-equivalencies. We hypothesize that when a patient is stable, the choice of action has less impact on the final outcome compared to less stable states.

\section{Conclusion}
In the context of learning SVPs for MDPs, we propose a model-free algorithm that is a variant of TD learning. On both synthetic and real RL tasks, our algorithm discovers meaningful action near-equivalencies, while maintaining overall near-optimality across states. Though the theoretical guarantees only hold for DAG settings, the near-greedy action selection heuristic can be easily extended to more complex settings involving non-DAGs and function approximation. In practice, to improve convergence and near-optimality guarantees, one could encode temporal information into the states for a discrete state space, converting a single ground state with different histories (e.g., visit number) into different states, effectively making the MDP a DAG. Despite current limitations, our proposed framework represents an important step toward clinician/human-in-the-loop decision making. Such a framework, in which both optimality guarantees and action choices are provided, allows clinicians (and patients) to incorporate additional information when making treatment decisions. Though motivated in a healthcare setting, our approach could apply more broadly to other application domains involving humans-in-the-loop, such as intelligent tutoring systems or self-driving cars.

\section*{Acknowledgments}

This work was supported by the National Science Foundation (NSF award no. IIS-1553146) and the National Library of Medicine (NLM grant no. R01LM013325). The views and conclusions in this document are those of the authors and should not be interpreted as necessarily representing the official policies, either expressed or implied, of the National Science Foundation nor the National Library of Medicine. The authors would like to thank Satinder Singh, Brahmajee Nallamothu, Jessica Golbus, and members of the \href{https://wiens-group.engin.umich.edu/}{MLD3 group} for helpful discussions regarding this work, as well as the reviewers for constructive feedback.

\bibliography{ref}
\bibliographystyle{icml2020}


\clearpage
\appendix
\onecolumn

\icmltitlerunning{Clinician-in-the-Loop Decision Making: RL with Near-Optimal Set-Valued Policies --- Supplementary Materials}

\section{Near-Optimal SVP With Additive Near-Optimality}\label{appx:additive}
We can quantify the near-optimality of any given SVP $\pi$ by using a version of the performance difference lemma \citep{kakade2002approximately}.

\begin{theorem}
For any SVP $\pi$, if for every state $s \in \mathcal{S}$:
\[\max_{a\in \pi(s)} \big(Q^*(s,a^*) - Q^*(s,a)\big) \leq \epsilon \ ,\]
then 
\[ V^*(s) - V^\pi(s) \le \frac{\epsilon}{1-\gamma} \ .\]

\end{theorem}
\begin{proof}
Note that $V^*(s) \leq V^{\pi}(s)$ for any SVP $\pi$ (equality holds when the SVP $\pi$ corresponds to an optimal policy). Denoting $a^*=\pi^*(s)$, and $\bar{a}=\argmin_{a\in \pi(s)}Q^{\pi}(s,a)$. We can evaluate the difference between $V^*$ and $V^{\pi}$ for a particular state $s$:
\begin{align*}
    \quad V^*(s) - V^{\pi}(s) 
    &= \left[r(s,a^*) + \gamma \mathbb{E}_{s'|s,a^*} V^*(s')\right] 
     - \left[r(s,\bar{a}) + \gamma\mathbb{E}_{s'|s,\bar{a}} V^{\pi}(s')\right] \\
    &= \left[r(s,a^*) + \gamma \mathbb{E}_{s'|s,a^*} V^*(s')\right] 
     {\color{blue} - \left[r(s,\bar{a}) + \gamma \mathbb{E}_{s'|s,\bar{a}} V^*(s')\right]} \\
    & \ \ \ \ \hspace{12.2em}  {\color{blue} + \left[r(s,\bar{a}) + \gamma \mathbb{E}_{s'|s,\bar{a}} V^*(s')\right]}
     - \left[r(s,\bar{a}) + \gamma\mathbb{E}_{s'|s,\bar{a}} V^{\pi}(s')\right] \\
    & \ \ \ \ \ \text{(adding and subtracting the expressions in {\color{blue} blue})} \\
    &= \left[ Q^*(s,a^*) - Q^*(s,\bar{a}) \right] + \gamma \left[ \mathbb{E}_{s'|s,\bar{a}} \left(V^*(s') - V^{\pi}(s') \right) \right] \ .
\end{align*}
Suppose we are guaranteed that for every state $s$, we have the following bound on the action-value gap of actions in $\pi(s)$: 
\[\max_{a\in \pi(s)} \big(Q^*(s,a^*) - Q^*(s,a)\big) \leq \epsilon \ ,\] or equivalently,
\[\forall a\in \pi(s): Q^*(s,a^*) - Q^*(s,a) \leq \epsilon \ ,\] we can further simplify:
\[\ \quad V^*(s) - V^{\pi}(s) \leq \epsilon + \gamma \left[ \mathbb{E}_{s'|s,\bar{a}} \left(V^*(s') - V^{\pi}(s') \right) \right] \ .\]

Unrolling the recursive expression:
\begin{align*}
    \quad V^*(s) - V^{\pi}(s) 
    &\leq \epsilon + \gamma\epsilon + \gamma^2\left[ \mathbb{E}_{s''|s,\bar{a},s', \bar{a}'} \left(V^*(s'') - V^{\pi}(s'') \right) \right] \\
    &\leq \epsilon + \gamma\epsilon + \gamma^2\epsilon +\cdots \\
    &= \frac{\epsilon}{1-\gamma} \ . \qedhere
\end{align*}
\end{proof}

If we want the worst-case values of $\pi$ to be within a margin defined as some fraction $\zeta$ of the maximum magnitude optimal value, e.g.,
\[\max_{s} \left(V^*(s) - V^{\pi}(s) \right) \leq \zeta \|V^*\|_\infty \ ,\] 
then we can set \[\frac{\epsilon}{1-\gamma} = \zeta \|V^*\|_\infty \ , \]
which implies that the action-value gap should be upper bounded by $\epsilon=(1-\gamma)\zeta \|V^*\|_\infty$. In practice, once we learn $Q^*$ and $V^*$, we can construct the SVP as \[\pi(s) = \{ a : Q^*(s,a) \geq V^*(s) - \epsilon \}. \]

\section{Learning SVPs via an Exponential Action Space -- And Why It Does Not Work} \label{appx:exp-action-space}

Alternatively, one might reformulate the task of learning SVPs by considering an exponentially large action space, $\tilde{\mathcal{A}} = 2^{\mathcal{A}} \setminus \{\varnothing\}$. By applying standard approaches to an MDP with this new action space, one can learn a policy $\pi$ that maps each state to an element of $\tilde{\mathcal{A}}$, such that $\pi(s) = \tilde{a}$. Under this formulation, Q-values are defined over $\mathcal{S} \times \tilde{\mathcal{A}}$, which we denote $Q^{\pi}(s,\tilde{a})$. Consider the worst-case Q-values defined analogously to \cref{def:worst-case-vf}: \[ Q^{\pi}(s,\tilde{a}) = \min_{a\in \tilde{a}} Q^{\pi}(s,a). \] Then, for any $\tilde{a} \in \tilde{\mathcal{A}}$, we have \[ Q^{\pi}(s,\tilde{a}) = \min_{a\in \tilde{a}} Q^{\pi}(s,a) \leq \max_{a\in\tilde{a}} Q^{\pi}(s,a), \]
and since $Q^{\pi}(s,a) = Q^{\pi}(s,\{a\})$, there exists an $\tilde{a}^*$ such that
\[ Q^{\pi}(s,\tilde{a}) \leq Q^{\pi}(s,\tilde{a}^*) \text{ where } \tilde{a}^* = \left\{ \argmax_{a\in\tilde{a}} Q^{\pi}(s,a) \right\} . \]
Intuitively, selecting the best action in $\tilde{a}$ is always \textit{no worse} than selecting the worst action in $\tilde{a}$. This suggests that for any non-singleton set action $\tilde{a}$, we can always find a singleton set action $\tilde{a}^*$ that is better. Thus, this formulation results in trivial SVPs and does not discover near-equivalent actions. To yield meaningful solutions, one would require additional constraints.

\section{Example: Non-Existence of Near-Greedy SVP Fixed-Point}\label{appx:non-existence}

Recall the near-greedy fixed-point equation: 
\begin{align*}
    \pi(s) = \{a : Q^{\pi}(s,a) \geq (1-\zeta)V^{*}(s)\} \text{ where } 
    Q^{\pi}(s,a) = r(s,a) + \gamma \mathbb{E}_{s'|s,a} \left[ \min_{a'\in \pi(s')} \{Q^{\pi}(s',a')\} \} \right] \label{eqn:near-greedy}
\end{align*}
Consider the MDP in \cref{fig:mdp-example} with two non-terminal states $\{s_1, s_2\}$ and two actions $\{L,R\}$. Let $\gamma =0.9$, $\zeta =0.2$. Here, $V^*=[0.9, 1]$. There are two candidate SVPs, both of which fail to satisfy the near-greedy fixed-point equation. 

\begin{itemize}
    \item Suppose $\pi(s_1) = \{R\}$, $\pi(s_2)=\{R\}$. Then $Q^{\pi}(s_2, L)=0.81 > (1-\zeta)V^*(s_2)$, meaning that $L$ is a near-optimal action at $s_2$ but not included in $\pi(s_2)$.
    \item Suppose $\pi(s_1)=\{R\}$, $\pi(s_2)=\{L,R\}$. Then the worst-case $Q^{\pi}(s_2, L)=0$ because the agent falls into a cycle in the worst case, and thus $L$ is not a near-optimal action but is included in $\pi(s_2)$.
\end{itemize}

\begin{figure}[h]
    \centering
    \includegraphics[width=0.4\linewidth]{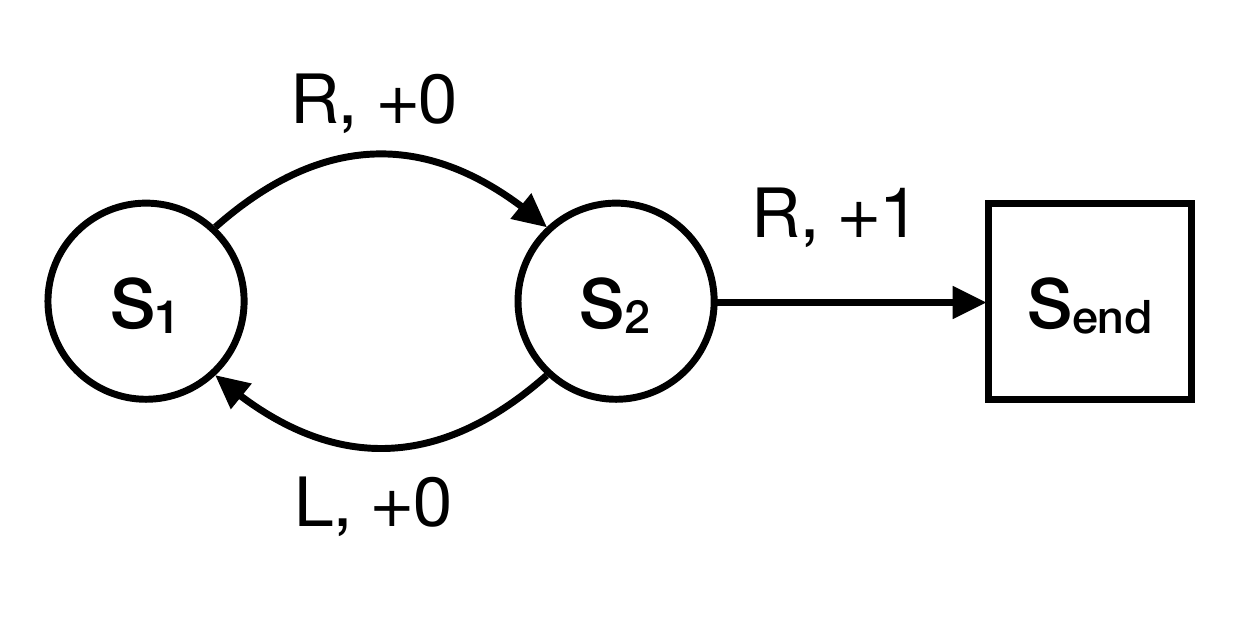}
    \caption{A three-state MDP with no near-greedy fixed-point solution when $\gamma = 0.9$ and $\zeta=0.2$.}
    \label{fig:mdp-example}
\end{figure}

\section{More on the Conservative Heuristic}\label{appx:conservative}
\begin{theorem}
The conservative $\zeta$-optimal SVP exists and is unique for any MDP with non-negative rewards.
\end{theorem}
\begin{proof} 
In the conservative heuristic, there is no recursive relationship between policy $\pi$ and its value function $V^{\pi}$ or $Q^{\pi}$. The policy construction depends on the lower-bound action-value function $\check{Q}^{*}$, which computes an expectation over $V^{*}$ and immediate rewards $r$, and is thus unique, and so is $\pi$. 

To show that $\pi$ is a valid SVP, we will show that the optimal action at every state is always included in $\pi(s)$ such that $\forall s \in \mathcal{S}, \pi(s) \neq \varnothing$. Consider the optimal action at state $s$, $a^* = \argmax_{a\in \mathcal{A}} Q^*(s,a)$ where $V^*(s) = Q^*(s,a^*)$, we have: 

\begin{align*}
    \check{Q}_{\zeta}^*(s,a^*)
    &= r(s,a^*) + \gamma \mathbb{E}_{s'|s,a^*} (1-\zeta) V^*(s') \\
    &= \zeta r(s,a^*) + (1-\zeta) [r(s,a^*) + \gamma \mathbb{E}_{s'|s,a^*} V^*(s')] \\
    &= \zeta r(s,a^*) + (1-\zeta) Q^*(s,a^*) \\
    &\geq (1-\zeta) Q^*(s,a^*) \\
    &= (1-\zeta) V^*(s). \qedhere
\end{align*}
\end{proof}


Since the conservative heuristic calculates an expectation over $V^*$ and $r(s,a)$ and does not involve any recursive relationship, after learning $Q^*$ (and thus $V^*$), we can apply a standard stochastic approximation algorithm with provable convergence guarantees \citep{robbins1951}. While the conservative heuristic has good theoretical properties, in \cref{sec:algo-behavior} we observe that it does not discover as many near-optimal actions compared to near-greedy (due to it being conservative). 


\section{Convergence Analysis for the Near-Greedy TD Algorithm\texorpdfstring{ (\cref{alg:near-greedy})}{}}
\label{appx:near-greedy-convergence-proof}

\subsection{Contraction}\label{appx:contraction}

For the case of a general MDP (possibly non-DAG), we refer to the convergence proofs of TD methods such as Q-learning and expected SARSA, which have been extensively studied in the tabular setting for problems with discrete state and action spaces \citep{watkins1992q,melo2001convergence,van2009expectedSARSA}. For Q-learning, given bounded rewards, $Q$ converges to the optimal value function $Q^{*}$, i.e., $Q(s,a) \simeq Q^{*}(s,a)$ for all $s\in \mathcal{S}$, $a \in \mathcal{A}$ with probability $1$, under regular conditions for stochastic approximation: each $(s,a)$ is updated infinitely many times, $\sum_{t} \alpha_{t}=\infty$, and $\sum_{t} \alpha_{t}^{2}<\infty$. One of the key steps in the proof involves showing that the update operator $\mathsf{H}$ is a contraction with respect to sup-norm \citep{melo2001convergence}:
\begin{align*}
    & \text{Update operator:} \\
    & (\mathsf{H}Q)(s,a) = r(s,a) + \gamma \mathbb{E}_{s'|s,a} \max_{a'\in\mathcal{A}} Q(s',a') \\
    & \text{based on the Bellman optimality equation, and} \\
    & \|\mathsf{H}Q_1 - \mathsf{H}Q_2\|_{\infty} \leq \gamma \|Q_1 - Q_2\|_{\infty}.
\end{align*}

Since the proposed algorithms have the same structure as TD learning, ideally we would have the same convergence guarantees. Consider the following update operator for the near-optimal TD algorithm:
\begin{align*}
    & (\mathsf{H}Q)(s,a) = r(s,a) + \gamma \mathbb{E}_{s'|s,a} \min_{a'\in\pi(s')} Q(s',a') \\
    & \text{where } \pi(s') = \{ a': Q^{\pi}(s',a') \geq (1-\zeta) V^{*}(s') \}.
\end{align*}

In an attempt to show that the update operator is a contraction, we can manipulate $\|\mathsf{H}Q_1 - \mathsf{H}Q_2\|_{\infty}$ in a similar way:
\begin{align*}
    \quad \|\mathsf{H}Q_1 - \mathsf{H}Q_2\|_{\infty}
    &= \max_{s,a} \left| \left(r(s,a) + \gamma \mathbb{E}_{s'|s,a} \min_{a_1'\in\pi_1(s')} Q_1(s',a_1') \right) - \left(r(s,a) + \gamma \mathbb{E}_{s'|s,a} \min_{a_2'\in\pi_2(s')} Q_2(s',a_2') \right) \right| \\
    &= \max_{s,a} \left| \gamma \mathbb{E}_{s'|s,a} \left[ \min_{a_1'\in\pi_1(s')} Q_1(s',a_1') - \min_{a_2'\in\pi_2(s')} Q_2(s',a_2') \right] \right| \\
    &\leq \max_{s,a} \gamma \mathbb{E}_{s'|s,a} \left| \min_{a_1'\in\pi_1(s')} Q_1(s',a_1') - \min_{a_2'\in\pi_2(s')} Q_2(s',a_2') \right| \\
    &\leq \gamma \max_{s'} \left| \min_{a_1'\in\pi_1(s')} Q_1(s',a_1') - \min_{a_2'\in\pi_2(s')} Q_2(s',a_2') \right| \\
    &\leq \gamma \max_{s'} \left| V^*(s') - (1-\zeta)V^*(s') \right| \\
    &= \gamma \zeta \max_{s'} V^{*}(s') 
    = \gamma \zeta \|V^{*}\|_{\infty}. 
\end{align*}

With this loose upper bound, the update operator is not necessarily a contraction, suggesting that the algorithm might not converge for a general MDP.

\subsection{Convergence Proof for DAG MDPs}
We first state a key result from martingale theory that we will use:

\begin{theorem}[Martingale Convergence Theorem \citep{williams1991probability}]
Consider $\{M_n\}_{n \in \mathbb{N}}$ as martingale in $\mathbb{R}^d$ with
\[\sum_{n \ge 0} \mathbb{E}\left[ \|M_{n+1} - M_n\|^2 | \mathcal{F}_n \right] < \infty\]
then there exists a random variable $M_\infty \in \mathbb{R}$ such that $\|M_\infty\| < \infty$ almost surely and $M_n \rightarrow_{n\rightarrow \infty} M_{\infty}$ almost surely.
\label{thm:mart_conv}
\end{theorem}

Using this standard result, we can show the following convergence result:
\addtocounter{theorem}{-4}
\begin{theorem}[restated] \label{thm:greedy_converge_app}
The near-greedy TD algorithm (\cref{alg:near-greedy}) converges to the unique solution if the MDP is a DAG with non-negative rewards, under the same conditions for regular TD learning: rewards have bounded variance, each $(s,a)$ is updated infinitely many times, $\sum_{t} \alpha_{t}=\infty$, and $\sum_{t} \alpha_{t}^{2}<\infty$ for each $(s,a)$ \citep{watkins1992q,melo2001convergence}. 
\end{theorem}
\addtocounter{theorem}{+4}

\begin{proof}
Given the DAG MDP, we use $H$ to denote the maximum number of steps (depth of the topological sort tree) and $(s_h, a_h)$ to denote a state-action pair for a particular state $s_h$ at step $h$. We use $Q_t(s_h,a_h)$ to denote the Q-value estimate after episode $t$ in \cref{alg:near-greedy}. In addition, we overload the notation $Q(h)$ to refer to the vector containing Q-values of all state-action pairs $\mathcal{S}(h) \times \mathcal{A}(h)$ at step $h$.

From \cref{thm:greedy_exist} we know that for a DAG MDP, the equation $\pi(s) = \{a : Q^{\pi}(s,a) \geq (1-\zeta)V^{*}(s)\}$ has a unique fixed point solution, which we denote $\pi^{\zeta}$ and its worst-case value function as $Q^{\zeta}$. Furthermore, we define the following:
\begin{align*}
    \overline{a}(s_h) &= \argmin_{a \in \pi^{\zeta}(s_h)} Q^{\zeta}(s_h,a) \\
    \underline{a}(s_h) &= \argmax_{a \notin \pi^{\zeta}(s_h)} Q^{\zeta}(s_h,a)
\end{align*}

Note that $Q^{\zeta}(s, \overline{a}(s)) \geq (1-\zeta)V^*(s) \geq Q^{\zeta}(s,\underline{a}(s))$. Intuitively, $\overline{a}$ gives us the worst-case action whose value will be used in the update / backup, whereas $\underline{a}$ is the best action outside the near-optimal action set for the given $\zeta$. 

We will prove the convergence of the near-greedy TD algorithm for DAG MDPs via backward induction over the episode steps $H, H-1, \ldots, 1$. 

\begin{adjustwidth}{1em}{0pt}
{\emph{Base step.}} For every terminal state $s_H$, the estimates are correct by initialization as $Q^{\zeta}(s,a) = 0$ and $\pi^{\zeta}(s) = \mathcal{A}$ trivially. Therefore, for all $(s_H, a_H)$ and $\epsilon \ge 0$, there exists $t_{\epsilon} \ge 0$, such that, for all $t \ge t_{\epsilon}$, $\|Q_t(H) - Q^{\zeta}(H)\|_\infty \le \epsilon$ where $Q(H)$ is the vector containing Q-values of state-action pairs at step $H$.

{\emph{Inductive step.}} Assume that $Q_t$ for all state-action pairs in levels $\{h+1, \ldots, H\}$ converge to the true $Q^{\zeta}$ almost surely. In other words, other than sequences of measure $0$, under all possible updates, we have $Q_t(s_j,a_j) \rightarrow Q^{\zeta}(s_j,a_j)$ for all $j \ge h+1$. This guarantees that for all $(s_{h+1},a_{h+1})$, for every $\epsilon > 0$, there exists $t_\epsilon > 0$ such that, for all $t \ge t_\epsilon$, $\|Q_t(h+1) - Q^{\zeta}(h+1)\|_\infty \le \epsilon$. For notational convenience in the inductive step, we use $(s,a)$ to denote state-action pairs at step $h$ and $(s',a')$ to denote state-action pairs at step $h+1$. 

Let $\Delta_1(s') = Q^{\zeta}(s',\overline{a}(s')) - Q^{\zeta}(s',\underline{a}(s'))$ and $\Delta_2(s') = \max_{a \in \pi^{\zeta}(s')} Q^{\zeta}(s',a) - Q^{\zeta}(s',\overline{a}(s'))$. Note that, if we pick $\epsilon < \tfrac{1}{2}\min_{s'} (\Delta_1(s'), \Delta_2(s'))$, then convergence implies that, for each state $s'$, after some episode $t_0$, a constant action $\overline{a}(s')$ is used in the near-greedy update of Q-values at step $h$.

Consider the sequence of Q-values $\{Q_t(h)\}_{t \in \mathbb{N}}$. Let $\mathcal{F}_{th}$ denote the history of the algorithm till step $h$ of episode $t$. In our proof, we consider the updates made to $Q(h)$ after $t_0$ with $Q_{t_0}(h)$ as its initialization for our analysis. This reduces the proof structure to a simple stochastic approximation based argument where the constant near-greedy action is used while bootstrapping for any state $s'$. At any such episode $t > t_0$, the algorithm makes an update of the following form to $Q_t(h)$:
\begin{align*}
    Q_{t+1}(s,a) = \left\{
	\begin{array}{ll}
		Q_t(s,a)  & \mbox{if } (s,a)_{th} \ne (s,a) \\
		(1-\alpha_{th})Q_t(s,a) + \alpha_{th} [r_{th} + \gamma Q_t(s'_{t},\overline{a}(s'_{t}))] & \mbox{if } (s,a)_{th} = (s,a)
	\end{array}
\right.
\end{align*}
We can rewrite the bootstrapping update as: 
\begin{align}
    Q_{t+1}(s,a) = {} & \underbrace{(1-\alpha_{th})Q_t(s,a) + \alpha_{th} \mathbb{E}_{r,s'}[r + \gamma Q_t(s',\overline{a}(s'))]}_{\text{Bellman update}} + \underbrace{\alpha_{th} w_{th}}_{\text{noise term}}
    \label{eq:Q-update}
\end{align}
where $w_{th} = [r_{th} + \gamma Q_t(s'_{t},\overline{a}(s'_{t}))] - \mathbb{E}_{r,s'}[r + \gamma Q_t(s',\overline{a}(s'))]$. We now analyze these two components of the update separately.

\begin{adjustwidth}{1em}{0pt}
{\emph{Bellman update.}} First note that $\mathbb{E}_{r,s'}[(r + \gamma Q_t(s',\overline{a}(s')))^2] < \infty$ by using the assumption $\mathbb{E}[r^2] < \infty$ and the inductive assumption on $Q_t$. In the near-greedy TD algorithm, for each $(s,a)$, the updates are made using step size $\alpha_t$ such that $\sum_{t} \alpha_{t}=\infty$, and $\sum_{t} \alpha_{t}^{2}<\infty$. Using $\bar{Q}(s,a)$ to denote the noise-free update term in \cref{eq:Q-update}, for the Bellman update sequence, we have:
\begin{align*}
    \bar{Q}_{t+1}(s,a) - Q^{\zeta}(s,a) = {} & (1-\alpha_{th})(Q_t(s,a)-Q^{\zeta}(s,a)) + \alpha_{th} \gamma\mathbb{E}_{s'}[ Q_t(s',\overline{a}(s')) - Q^{\zeta}(s',\overline{a}(s'))] \\
    \le {} & (1-\alpha_{th})(Q_t(s,a)-Q^{\zeta}(s,a)) + \alpha_{th} \gamma\epsilon
\end{align*}
where the last step follows from the inductive assumption. Using the standard results from stochastic approximation \citep{robbins1951}, we can conclude that the deterministic error $\Pi_{t>t_0} (1-\alpha_{th})^2 (Q_{t_0}(s,a) - Q^{\zeta}(s,a))^2$ converges to 0 implying $\lim\,\sup_{t \rightarrow \infty} (\bar{Q}_t(s,a) - Q^{\zeta}(s,a))^2 \le C\epsilon$ for some constant $C$. As the chosen $\epsilon$ is arbitrary, by the sandwich theorem for limits, the error incurred via the Bellman update sequence converges to $0$ almost surely.

{\emph{Noise term.}} We will now argue that the noise sequence $\sum_{t>t_0} \alpha_{th} w_{th}$ also converges to $0$. Note that, $Z_t = \sum_{t>t_0} \alpha_{th} w_{th} \in \mathbb{R}^{\mathcal{S}(h)\times \mathcal{A}(h)}$ is a martingale sequence as $\mathbb{E}[w_{th}(s,a)|\mathcal{F}_{th}] = 0$. Further, again by the bounded variance assumption over rewards and the inductive assumption over $Q(h+1)$, we have 
\begin{align*}
    \sum_{t > t_0} \mathbb{E}\left[ \|Z_{t+1} - Z_{t}\|^2 | \mathcal{F}_{th} \right] = \sum_{t > t_0} \alpha_{th}^2 \mathbb{E}\left[ \|w_{th}\|^2 | \mathcal{F}_{th} \right] \le c \cdot \sum_{t > t_0} \alpha_{th}^2 \le \infty
\end{align*}
Now using \cref{thm:mart_conv} and the definition $Z_{t_0} = 0$, we can conclude that the martingale converges to $0$ almost surely.
\end{adjustwidth}

We know that for two sequences of random variables $X_n$ and $Y_n$, if $X_n \rightarrow X$ and $Y_n \rightarrow Y$ almost surely, then $X_n + Y_n \rightarrow X+Y$ almost surely. Combining the two parts, we get $\|Q(h) - Q^{\zeta}(h)\|_\infty \rightarrow 0$ almost surely. This completes the inductive step. 

\end{adjustwidth}
By induction, this proves the desired convergence result.
\end{proof}

\section{Comparisons to the Mixed-Integer Programming (MIP) Baseline} \label{appx:MIP-baseline}

\citet{fard2011non} proposed a mixed-integer programming formulation for solving the maximal-size SVP in a finite-horizon tabular planning problem. The optimization problem jointly solves for the worst-case values $V$ and a binary representation of SVP $\pi$, where $\Pi(s,a) = \llbracket a \in \pi(s) \rrbracket$ is $1$ if $a$ is an element of $\pi(s)$, and $0$ otherwise. There are a total of $|\mathcal{S}|(|\mathcal{A}|+1)$ decision variables and $|\mathcal{S}| (|\mathcal{A}|+2)$ constraints. The formulation is reproduced below; see \citet{fard2011non} for more details. 

\vspace{-1.25em}
\scalebox{0.8}{
\begin{minipage}{1.25\linewidth}
\begin{align*}
    \max_{V, \Pi} \big[ \mu^{\intercal} V + & (V_{\max} - V_{\min}) e_s^{\intercal} \pi e_a \big]  \text{ subject to } \\
    & V(s) \geq (1-\zeta) V^*(s) && \forall s\in\mathcal{S} \\
    & \textstyle \sum_{a\in\mathcal{A}} \Pi(s,a) > 0 && \forall s\in\mathcal{S} \\
    & V(s) \leq r(s,a) + \gamma \sum_{s'\in\mathcal{S}} p(s'|s,a)V(s') + V_{\max}(1-\Pi(s,a)) && \forall s\in\mathcal{S}, \forall a\in\mathcal{A}
\end{align*}
\end{minipage}
}
\vspace{-0.6em}

Since the MIP approach requires knowledge of the MDP model, we implemented a dynamic programming based approach with the near-greedy heuristic, namely near-greedy value iteration (VI). We applied these two algorithms on simple environments where the MIP solution is tractable. On \textsf{Chain-5} with $\gamma=0.9$, where the underlying MDP is a DAG (\cref{fig:chain5-MIP}), near-greedy VI converged for all values of $\zeta$. The SVPs learned by both approaches satisfy near-optimality with respect to the given $\zeta$, as shown by the worst-case near-optimality percentages. For $\zeta \geq 0.1$, the SVP included all actions at every state. Even though near-greedy VI is not explicitly maximizing the policy size (unlike the MIP approach, which includes policy size as part of its objective function), for many of the cases it still finds an SVP solution with maximal size, or close to the maximal-size solution as found by MIP (when $\zeta=0.03$ and $0.04$ on this problem). On a non-DAG environment, \textsf{CyclicChain-5} with $\gamma=0.9$ (\cref{fig:cycchain5-MIP}), near-greedy VI did not converge for $0.2 \leq \zeta < 1$ (when a near-optimal SVP should only include the two `right' actions but no `left' actions). This is consistent with what we observed in \cref{fig:converge-a}. On this problem, when near-greedy VI does converge ($\zeta \leq 0.1$, which is a suitable range of values if one aims to learn \textit{close-to-optimal} behavior), it consistently finds the same maximal-size SVP as MIP. Compared to a model-based approach based on exhaustive search, our proposed near-greedy heuristic identifies SVP solutions that achieve good worst-case near-optimality and similar average policy sizes, despite the fact that we do not explicitly optimize for the size of the SVP.

\begin{figure}
\begin{center}
\renewcommand*{\arraystretch}{1.2}
\scalebox{0.85}{
\begin{tabular}{lccccc}
\toprule
& \multicolumn{2}{c}{\textbf{Near-greedy VI}} && \multicolumn{2}{c}{\textbf{MIP}} \\
\cmidrule{2-3} \cmidrule{5-6}
$\zeta$ & policy profile & \multicolumn{3}{c}{\makecell{\small average \\ \small policy size}} & policy profile \\
\midrule
$0$ & 
\adjustimage{scale=0.5,valign=m,padding=0em 2pt}{{fig_MIP/_chain_5_near-greedy_zeta_0.0}.pdf} 
{\small(100.0\%)} 
& 1 && 1
& 
{\small(100.0\%)} 
\adjustimage{scale=0.5,valign=m,padding=0em 2pt}{{fig_MIP/_chain_5_MIP_zeta_0.0}.pdf} 
\\
$0.01$ & 
\adjustimage{scale=0.5,valign=m,padding=0em 2pt}{{fig_MIP/_chain_5_near-greedy_zeta_0.01}.pdf} 
{\small(\; 99.0\%)} 
& 1.5 && 1.5 
& 
{\small(\; 99.0\%)} 
\adjustimage{scale=0.5,valign=m,padding=0em 2pt}{{fig_MIP/_chain_5_MIP_zeta_0.01}.pdf}
\\
$0.02$ & 
\adjustimage{scale=0.5,valign=m,padding=0em 2pt}{{fig_MIP/_chain_5_near-greedy_zeta_0.02}.pdf} 
{\small(\; 98.1\%)} 
& 1.75 && 1.75
& 
{\small(\; 98.0\%)} 
\adjustimage{scale=0.5,valign=m,padding=0em 2pt}{{fig_MIP/_chain_5_MIP_zeta_0.02}.pdf} 
\\
$0.03$ & 
\adjustimage{scale=0.5,valign=m,padding=0em 2pt}{{fig_MIP/_chain_5_near-greedy_zeta_0.03}.pdf} 
{\small(\; 97.1\%)} 
& 2 && 2.25 
& 
{\small(\; 97.0\%)} 
\adjustimage{scale=0.5,valign=m,padding=0em 2pt}{{fig_MIP/_chain_5_MIP_zeta_0.03}.pdf} 
\\
$0.04$ & 
\adjustimage{scale=0.5,valign=m,padding=0em 2pt}{{fig_MIP/_chain_5_near-greedy_zeta_0.04}.pdf} 
{\small(\; 96.2\%)} 
& 2.25 && 2.5
& 
{\small(\; 96.1\%)} 
\adjustimage{scale=0.5,valign=m,padding=0em 2pt}{{fig_MIP/_chain_5_MIP_zeta_0.04}.pdf} 
\\
$0.05$ & 
\adjustimage{scale=0.5,valign=m,padding=0em 2pt}{{fig_MIP/_chain_5_near-greedy_zeta_0.05}.pdf}
{\small(\; 95.2\%)} 
& 2.75 && 2.75
& 
{\small(\; 95.2\%)} 
\adjustimage{scale=0.5,valign=m,padding=0em 2pt}{{fig_MIP/_chain_5_MIP_zeta_0.05}.pdf} 
\\
$0.1$ & 
\adjustimage{scale=0.5,valign=m,padding=0em 2pt}{{fig_MIP/_chain_5_near-greedy_zeta_0.1}.pdf}
{\small(\; 90.2\%)} 
& 4 && 4
& 
{\small(\; 90.2\%)} 
\adjustimage{scale=0.5,valign=m,padding=0em 2pt}{{fig_MIP/_chain_5_MIP_zeta_0.1}.pdf} 
\\
$0.2$ & 
\adjustimage{scale=0.5,valign=m,padding=0em 2pt}{{fig_MIP/_chain_5_near-greedy_zeta_0.2}.pdf}
{\small(\; 90.2\%)} 
& 4 && 4
& 
{\small(\; 90.2\%)} 
\adjustimage{scale=0.5,valign=m,padding=0em 2pt}{{fig_MIP/_chain_5_MIP_zeta_0.2}.pdf} 
\\
$1$ & 
\adjustimage{scale=0.5,valign=m,padding=0em 2pt}{{fig_MIP/_chain_5_near-greedy_zeta_1.0}.pdf} 
{\small(\; 90.2\%)} 
& 4 && 4
& 
{\small(\; 90.2\%)} 
\adjustimage{scale=0.5,valign=m,padding=0em 2pt}{{fig_MIP/_chain_5_MIP_zeta_1.0}.pdf} 
\\
\bottomrule
\end{tabular}
}
\end{center}
\caption{SVPs learned by the near-greedy and MIP algorithms on \textsf{Chain-5} at different $\zeta$s. Parenthesized percentages denote the worst-case near-optimality.} \label{fig:chain5-MIP}
\end{figure}

\begin{figure}
\begin{center}
\renewcommand*{\arraystretch}{1.2}
\scalebox{0.85}{
\begin{tabular}{lccccc}
\toprule
& \multicolumn{2}{c}{\textbf{Near-greedy VI}} && \multicolumn{2}{c}{\textbf{MIP}} \\
\cmidrule{2-3} \cmidrule{5-6}
$\zeta$ & policy profile & \multicolumn{3}{c}{\makecell{\small average \\ \small policy size}} & policy profile \\
\midrule
$0$ & 
\adjustimage{scale=0.5,valign=m,padding=0em 2pt}{{fig_MIP_cyc/_cycchain_5_near-greedy_zeta_0.0}.pdf} 
{\small(100.0\%)} 
& 1 && 1
& 
{\small(100.0\%)} 
\adjustimage{scale=0.5,valign=m,padding=0em 2pt}{{fig_MIP_cyc/_cycchain_5_MIP_zeta_0.0}.pdf} 
\\
$0.01$ & 
\adjustimage{scale=0.5,valign=m,padding=0em 2pt}{{fig_MIP_cyc/_cycchain_5_near-greedy_zeta_0.01}.pdf} 
{\small(100.0\%)} 
& 1.5 && 1.5 
& 
{\small(100.0\%)} 
\adjustimage{scale=0.5,valign=m,padding=0em 2pt}{{fig_MIP_cyc/_cycchain_5_MIP_zeta_0.01}.pdf}
\\
$0.02$ & 
\adjustimage{scale=0.5,valign=m,padding=0em 2pt}{{fig_MIP_cyc/_cycchain_5_near-greedy_zeta_0.02}.pdf} 
{\small(\; 98.1\%)} 
& 1.75 && 1.75
& 
{\small(\; 98.9\%)} 
\adjustimage{scale=0.5,valign=m,padding=0em 2pt}{{fig_MIP_cyc/_cycchain_5_MIP_zeta_0.02}.pdf} 
\\
$0.03$ & 
\adjustimage{scale=0.5,valign=m,padding=0em 2pt}{{fig_MIP_cyc/_cycchain_5_near-greedy_zeta_0.03}.pdf} 
{\small(\; 97.9\%)} 
& 1.75 && 1.75 
& 
{\small(\; 98.9\%)} 
\adjustimage{scale=0.5,valign=m,padding=0em 2pt}{{fig_MIP_cyc/_cycchain_5_MIP_zeta_0.03}.pdf} 
\\
$0.04$ & 
\adjustimage{scale=0.5,valign=m,padding=0em 2pt}{{fig_MIP_cyc/_cycchain_5_near-greedy_zeta_0.04}.pdf} 
{\small(\; 96.8\%)} 
& 2 && 2
& 
{\small(\; 96.8\%)} 
\adjustimage{scale=0.5,valign=m,padding=0em 2pt}{{fig_MIP_cyc/_cycchain_5_MIP_zeta_0.04}.pdf} 
\\
$0.05$ & 
\adjustimage{scale=0.5,valign=m,padding=0em 2pt}{{fig_MIP_cyc/_cycchain_5_near-greedy_zeta_0.05}.pdf}
{\small(\; 96.8\%)} 
& 2 && 2
& 
{\small(\; 96.8\%)} 
\adjustimage{scale=0.5,valign=m,padding=0em 2pt}{{fig_MIP_cyc/_cycchain_5_MIP_zeta_0.05}.pdf} 
\\
$0.1$ & 
\adjustimage{scale=0.5,valign=m,padding=0em 2pt}{{fig_MIP_cyc/_cycchain_5_near-greedy_zeta_0.1}.pdf}
{\small(\; 96.8\%)} 
& 2 && 2
& 
{\small(\; 96.8\%)} 
\adjustimage{scale=0.5,valign=m,padding=0em 2pt}{{fig_MIP_cyc/_cycchain_5_MIP_zeta_0.1}.pdf} 
\\
$0.2$ & 
(did not converge)
& - && 2
& 
{\small(\; 96.8\%)} 
\adjustimage{scale=0.5,valign=m,padding=0em 2pt}{{fig_MIP_cyc/_cycchain_5_MIP_zeta_0.2}.pdf} 
\\
$1$ & 
\adjustimage{scale=0.5,valign=m,padding=0em 2pt}{{fig_MIP_cyc/_cycchain_5_near-greedy_zeta_1.0}.pdf} 
{\small(\; 16.1\%)} 
& 4 && 4
& 
{\small(\; 16.1\%)} 
\adjustimage{scale=0.5,valign=m,padding=0em 2pt}{{fig_MIP_cyc/_cycchain_5_MIP_zeta_1.0}.pdf} 
\\
\bottomrule
\end{tabular}
}
\end{center}
\caption{SVPs learned by the near-greedy and MIP algorithms on \textsf{CyclicChain-5} at different $\zeta$s. Parenthesized percentages denote the worst-case near-optimality.} \label{fig:cycchain5-MIP}
\end{figure}

\newpage
\section{Policy Evaluation for SVPs}
\label{appx:policy-evaluation}
For completeness, we describe the policy evaluation algorithm for an SVP and show that the update is a contraction, thereby, guaranteeing convergence.

Given an SVP $\pi$, the value functions are defined as:
\begin{align*}
    V^\pi(s) = {} & \min_{a \in \pi(s)} Q^\pi(s,a) \\
    Q^\pi(s,a) = {} & \mathbb{E}_{r,s'}\left[r + \gamma V^\pi(s') \right]
\end{align*}
The value function for any given policy $\pi$ can be evaluated easily via a simple modification of iterative policy evaluation algorithm for deterministic/stochastic policies \citep{sutton2018RL}:
\\
\begin{figure}[h]
\begin{minipage}{1.\linewidth}%
\begin{algorithm}[H]
  \caption{Iterative policy evaluation for set-valued policies}\label{alg:pol_eval}
  \begin{algorithmic}[1]
    \State \textbf{Input:} SVP $\pi$
    \State \textbf{Initialize} $Q(s,a)=0$ for all $s\in \mathcal{S}, a \in \mathcal{A}$
    \Repeat
        \State $\Delta \leftarrow 0$
        \For{each $s,a \in \mathcal{S} \times \mathcal{A}$}
            \State $Q'(s,a) = \mathbb{E}_{r,s'}\left[r + \gamma \min_{a' \in \pi(s')} Q(s',a') \right]$
            \State $\Delta \leftarrow \max(\Delta, |Q'(s,a) - Q(s,a)|)$
        \EndFor
        \State $Q \leftarrow Q'$
    \Until{$\Delta < \theta$}\\
    \Return $Q$
  \end{algorithmic}
\end{algorithm}
\end{minipage}
\end{figure}

We now show that the update in \cref{alg:pol_eval} is a contraction:
\begin{lemma}
For any pair of action-value functions $Q_1$ and $Q_2$, and a given policy $\pi$, we have:%
\begin{align*}
    \|\mathcal{T} Q_1 - \mathcal{T} Q_2\|_\infty \le \gamma\|Q_1 - Q_2\|_\infty
\end{align*}
\end{lemma}
\begin{proof}
For any $s,a \in \mathcal{S} \times \mathcal{A}$, we have:%
\begin{align*}
    \quad \left|(\mathcal{T} Q_1)(s,a) - (\mathcal{T} Q_2)(s,a)\right| 
    & {}  = \left|\mathbb{E}_{r,s'}\left[r + \gamma V^\pi_1(s') \right] - \mathbb{E}_{r,s'}\left[r + \gamma V^\pi_2(s') \right]\right| \\
    & {} = \gamma \left|\mathbb{E}_{s'}\left[ \min_{a_1 \in \pi(s')} Q_1(s',a_1)  - \min_{a_2 \in \pi(s')} Q_2(s',a_2)\right]\right| \\
    & {} \le \gamma \max_{s \in \mathcal{S}} \left|\left[ \min_{a_1 \in \pi(s)} Q_1(s,a_1)  - \min_{a_2 \in \pi(s)} Q_2(s,a_2)\right]\right|\\ 
    & {} \le \gamma \|Q_1 - Q_2\|_\infty. \qedhere
\end{align*}
\end{proof}

The contraction lemma further implies that \cref{alg:pol_eval} converges to the unique fixed point of the value function of the policy $\pi$. As the update is a straightforward modification of the usual Bellman operator, we can implement/analyze a fitted policy evaluation algorithm for SVPs as well.

\newpage
\section{Clinical Task Details}\label{appx:features}
\newcommand{\tabitem}{~~\llap{\textbullet}~~}

Following \citet{komorowski2018AI_Clinician}, we extracted 48 physiological features (\cref{tab:mimic-features}) to represent each patient. 

\begin{table}[h!]
    \centering
    \caption{The 48 physiological features}
    \label{tab:mimic-features}
    \scalebox{0.8}{
    \begin{tabular}{l}
    \toprule
        \textbf{Demographics/Static} \\
        Source tables: \texttt{PATIENTS}, \texttt{ADMISSIONS}, \\\texttt{ICUSTAYS}, \texttt{CHARTEVENTS}, \texttt{elixhauser\_quan} \\
        \tabitem Shock Index \\
        \tabitem Elixhauser \\
        \tabitem SIRS \\
        \tabitem Gender \\
        \tabitem Re-admission \\
        \tabitem GCS - Glasgow Coma Scale \\
        \tabitem SOFA - Sequential Organ Failure Assessment \\
        \tabitem Age \\
    \midrule
        \textbf{Lab Values} \\
        Source tables: \texttt{CHARTEVENTS}, \texttt{LABEVENTS} \\
        \tabitem Albumin \\
        \tabitem Arterial pH \\
        \tabitem Calcium \\
        \tabitem Glucose \\
        \tabitem Hemoglobin \\
        \tabitem Magnesium \\
        \tabitem PTT - Partial Thromboplastin Time \\
        \tabitem Potassium \\
        \tabitem SGPT - Serum Glutamic-Pyruvic Transaminase \\
        \tabitem Arterial Blood Gas \\
        \tabitem BUN - Blood Urea Nitrogen \\
        \tabitem Chloride \\
        \tabitem Bicarbonate \\
        \tabitem INR - International Normalized Ratio \\
        \tabitem Sodium \\
        \tabitem Arterial Lactate \\
        \tabitem CO2 \\
        \tabitem Creatinine \\
        \tabitem Ionised Calcium \\
        \tabitem PT - Prothrombin Time \\
        \tabitem Platelets Count \\
        \tabitem SGOT - Serum Glutamic-Oxaloacetic Transaminase \\
        \tabitem Total bilirubin \\
        \tabitem White Blood Cell Count \\
    \midrule
        \textbf{Vital Signs} \\
        Source tables: \texttt{CHARTEVENTS} \\
        \tabitem Diastolic Blood Pressure \\
        \tabitem Systolic Blood Pressure \\
        \tabitem Mean Blood Pressure \\
        \tabitem PaCO2 \\
        \tabitem PaO2 \\
        \tabitem FiO2 \\
        \tabitem PaO/FiO2 ratio \\
        \tabitem Respiratory Rate \\
        \tabitem Temperature (Celsius) \\
        \tabitem Weight (kg) \\
        \tabitem Heart Rate \\
        \tabitem SpO2 \\
    \midrule
        \textbf{Intake and Output Events} \\
        Source tables: \texttt{INPUTEVENTS\_CV}, \texttt{INPUTEVENTS\_MV}, \\ \texttt{OUTPUTEVENTS} \\
        \tabitem Fluid Output - 4 hourly period \\
        \tabitem Total Fluid Output \\
        \tabitem Mechanical Ventilation \\
    \midrule
        \tabitem Timestep \\
    \bottomrule
    \end{tabular}
    }
\end{table}


\end{document}